\newcommand{\longversion}[1]{#1}
\newcommand{\shortversion}[1]{}

\documentclass[a4paper, 10pt, conference]{ieeeconf}      

\IEEEoverridecommandlockouts                              

\usepackage{epsfig} 
\usepackage{mathptmx} 
\usepackage{times} 
\usepackage{amsmath} 
\usepackage{amssymb}  
\usepackage{algorithm}
\usepackage{booktabs}
\usepackage{multirow}

\usepackage{tikz}
\usepackage{siunitx}
\usepackage{algorithm,algorithmicx,algpseudocode,latexsym}

\usepackage{bibentry}
\newcommand{\RN}[1]{
  \textup{\uppercase\expandafter{\romannumeral#1}}
}

\graphicspath{{images/}}  
\DeclareGraphicsExtensions{.eps}

\newtheorem{theorem}{Theorem}
\newtheorem{definition}{Definition}

\newcommand{\comment}[1]  {\noindent \textcolor{red}{{\bf IS: }{}}}

\shortversion{
\renewcommand{\baselinestretch}{0.95}
}

\title{\LARGE \bf
Optimal Integrated Task and Path Planning and Its Application to Multi-Robot Pickup and Delivery
}
\author{Aman Aryan$^{1}$, Manan Modi$^{2}$, Indranil Saha$^3$, Rupak Majumdar$^4$ and Swarup Mohalik$^5$
\thanks{$^{1}$ Aman Aryan is with IIT Kanpur, India. {\tt\small aman.aryan0@gmail.com}}
\thanks{$^{2}$ Manan Modi is with Jupiter Money, India. {\tt\small  modimanann@gmail.com}}
\thanks{$^{3}$ Indranil Saha is with IIT Kanpur, India. {\tt\small isaha@cse.iitk.ac.in}}
\thanks{$^{4}$ Rupak Majumdar is with MPI-SWS, Germany. {\tt\small  rupak@mpi-sws.org}}
\thanks{$^{5}$ Swarup Mohalik is with Ericson Research, India.\  {\tt\small  swarup.kumar.mohalik@ericsson.com}}
}

\begin{document}

\maketitle
\thispagestyle{empty}
\pagestyle{empty}

\begin{abstract}
We propose a generic multi-robot planning mechanism that combines an optimal task planner and an optimal path planner to provide a scalable solution for complex multi-robot planning problems.
The Integrated planner, through the interaction of the task planner and the path planner, produces optimal collision-free trajectories for the robots.
We illustrate our general algorithm on an object pick-and-drop planning problem in a warehouse scenario where a group of robots is entrusted with moving objects from one location to another in the workspace. 
We solve the task planning problem by reducing it into an SMT-solving problem and employing the highly advanced \emph{SMT solver Z3} to solve it.
To generate collision-free movement of the robots, we extend the state-of-the-art algorithm \emph{Conflict Based Search with Precedence Constraints}  with several domain-specific constraints.
We evaluate our integrated task and path planner extensively on various instances of the object pick-and-drop planning problem and compare its performance  with a state-of-the-art multi-robot classical planner. 
Experimental results demonstrate that our planning mechanism can deal with complex planning problems and outperforms a state-of-the-art classical planner both in terms of computation time and the quality of the generated plan.
\end{abstract}

\section{Introduction}
\label{sec-intro}

A major component of the software controlling a robotic system is a \emph{planner} that guides the robots to safely move through their workspace and perform the designated tasks appropriately. 
A planner for an application involving
mobile robots needs to have two components: a \emph{task planner} that decides which tasks should be performed by which robots and in what order, and a \emph{path planner} that provides the collision-free trajectories to be followed by the robots to reach the locations to perform the tasks. The task planning and the path planning problems cannot be addressed entirely independently as the assignment of a task to a robot is directly related to the amount of effort the robot needs to invest in reaching the task locations. 

Consider a multi-robot application where a group of mobile robots is entrusted with the responsibility of delivering objects from one location to another in a workspace. The task assignment to the robots depends on the time required to traverse 
the distance between the initial locations of the robots and various task locations and the distance between the task locations when a robot has to perform multiple tasks. 
The traverse time 
between different locations depends on the collision-free optimal trajectories of the robots, which can only be obtained from a multi-robot path planner.

Two different approaches are, in general, employed to solve a multi-robot planning problem offline for a static environment. 
In the first approach, the multi-robot task assignment  and the path planning problems are formulated and solved as a monolithic problem  (e.g.,~\cite{CrosbyRP13,saha2014,honig2018conflict}).
In the second approach, the task assignment problem is solved based on a heuristic to measure the trajectory lengths approximately (e.g.,~\cite{GavranMS17,saha2022costs, turpin2013trajectory}).
As the task assignment is not carried out based on collision-free trajectories, a local collision avoidance strategy (e.g.~\cite{HennesCMT12}) is employed during the execution of the plan. 
The shortcoming of the first approach is that it either fails to provide a multi-robot trajectory with a guarantee on its optimality~\cite{CrosbyRP13}, or the algorithm that can produce an optimal plan takes a prohibitively large amount of time to compute the collision-free trajectories~\cite{saha2014,honig2018conflict}.
The second approach can find a plan quickly, but the generated plans are guaranteed to be neither collision-free nor optimal.

To bridge this gap, we design a scalable algorithm to generate optimal collision-free trajectories for multi-robot systems. 
The proposed algorithm works as follows. 
It first estimates the lengths of the trajectories between all locations of interest through which a robot may need to move. 
Based on the estimated trajectory lengths, the task planner generates a task assignment corresponding to optimal trajectories for the robots based on the estimated length of the trajectories between any two locations. 
The outcome of the task assignment is a sequence of locations to be visited by all the robots. 
In the second step, we generate collision-free trajectories for the robots to reach their designated locations in sequence by means of an optimal multi-robot path planner. 
If the cost of the trajectories obtained in this step is more than that of the trajectories obtained during task assignment, we look for another same-cost or a sub-optimal task assignment for which the cost of the collision-free trajectories obtained by solving the multi-robot path planning problem may be better than the collision-free trajectories obtained in the previous step. 
In this way, we alternate between the task planner and the path planner until we find a task assignment with optimal-cost collision-free trajectories.

We illustrate our general algorithm on an offline multi-agent pick-and-drop planning problem in a warehouse scenario where a group of robots 
move objects from one location to another in the workspace. Our problem statement is similar to \cite{liu2019task} except that we have defined a designated base location for robots to return after finishing the tasks. We transform the task-planning problem into an SMT-solving problem that incorporates many application-specific operational constraints and solve it using the Z3~\cite{Z3_Moura} solver. Additionally, we employ the existing optimal multi-robot path planning algorithm $\textsf{MLA*-CBS-PC}$~\cite{zhang2022multi} to accommodate the sequential goal locations for each robot, thereby serving as the optimal path planner. 

We have evaluated our algorithm extensively on various instances of the object pick-and-drop planning problem and compared the performance of our planner with a state-of-the-art multi-robot classical planner. 
Experimental results demonstrate that our planning mechanism can deal with complex planning problems and outperform the state-of-the-art classical planner ENHSP in terms of computation time and quality of the generated plan.

\longversion{
In summary, we make the following contributions. 
\begin{itemize}
     \item We provide a general multi-robot planning algorithm that induces an interaction between the task planner and the path planner to generate optimal collision-free trajectories for the robots to enable them to complete the mission successfully (Section~\ref{sec-algorithm}).
    \item We provide an SMT-based task planner for object pick-and-drop applications in a warehouse scenario. Our task planner is general enough to be able to incorporate many application-specific operational constraints (Section~\ref{sec-taskplanning}).
    \item We adapt the state-of-the-art graph-based multi-robot path planner $\textsf{MLA*-CBS-PC}$ \cite{zhang2022multi} to deal with a sequence of goal locations for each robot (Section~\ref{sec-motionplanning}) using plans generated from our task planner.
    \item We demonstrate the overall algorithm for multi-agent pickup and delivery application on predefined as well as randomly generated maps for various scenarios and compare it to the state-of-the-art classical planner ENHSP.
\end{itemize}
}

\section{Problem}
\label{sec-problem}

In this section, we define our problem formally and illustrate it with an example.

\subsection{Preliminaries}
        
\subsubsection{Workspace} The workspace, denoted by $\mathcal{W}$, is represented as a 2-D rectangular grid. 
We assume that the robots, as well as the task objects, occupy one grid block each at any time instance. 
Obstacles may occupy some of these grid blocks and thus cannot be used by the robots, tasks, or movement. Formally, the workspace is represented by a tuple $\langle L_X, L_Y, \Omega \rangle$, where $L_X$ and $L_Y$ denote the length and the width of the workspace, and $\Omega$ denotes the set of grid blocks that are occupied by obstacles.
        
\subsubsection{Robots} 
The set of robots is denoted by $\mathcal{R}$.
Each robot $r_i \in \mathcal{R}$ is defined as a tuple $\langle s_i, \Gamma_i, \Lambda_i, attributes_i \rangle$. The symbol $s_i$ denotes the start location of robot $r_i$. The symbols $\Gamma_i$ and $\Lambda_i$ denote the set of \emph{motion primitives} and \emph{action primitives} for robot $r_i$, respectively. 
To keep the exposition simple, we assume that each robot has five basic motion primitives: \textsf{move up}, \textsf{move down}, \textsf{move left}, \textsf{move right}, and \textsf{stay}. However, our methodology seamlessly applies to any complex set of motion primitives for a robot. The action primitives for a robot are application-specific. 
For example, for a pick-and-drop application, the robot has action primitives for \emph{picking up} and \emph{dropping off} an object. 
We assume that all of these primitives take a one-time step regardless of the robot's direction. 
Moreover, the motion and action primitives are deterministic, i.e., the application of a primitive to a robot in a state moves the robot to a unique next state. 
We denote by $attributes_i$ a set of attributes of robot $r_i$ that may be required depending upon the nature of the problem. For example, in a pick-and-drop example, an attribute for a robot could be the number of objects or the total amount of weight the robot can carry at once.

\subsubsection{Tasks} The set of tasks associated with a problem is denoted by $\mathcal{T}$. A task $t_i \in \mathcal{T}$ is defined as a tuple $\langle L_i, attributes_i \rangle$. Here, $L_i$ is a sequence of locations that need to be visited by a robot in the same order to complete the task. We denote by $attributes_i$ a set of attributes of the task that may be required for planning depending upon the nature of the problem. For example, a task $t_i$ may be associated with a deadline $d_i$; in that case, the last location in $L_i$ must be visited before $d_i$.
  
\subsubsection{Plan and Trajectory} 
We capture the behaviour of a robot in the workspace as a sequence of states. The state of robot $r_i$ at time step $t$ is denoted by $\sigma_i(t)$. 
Given a state $\sigma$ and a motion or action primitive $\nu$, the robot's next state $\sigma'$ is given by $\mathtt{next}(\sigma, \nu)$.

\begin{definition}[Plan] 
The plan for a robot $r_i$ is the sequence of motion and action primitives executed by the robot.
\end{definition}

\begin{definition}[Trajectory]
For robot $r_i$ with plan $\nu_i = (\nu_i(1), \nu_i(2), \ldots \nu_i(T_i))$, the trajectory is given by $\sigma_i = (\sigma_i(0), \sigma_i(1), \ldots, \sigma_i(T_i))$, where ${\sigma_i(0) = s_i}$ and for all ${i \in \{1, \ldots, T_i\}}$. \  {$\sigma_i(j) = {\mathtt{next}(\sigma_i(j-1), \nu_i(j))}$}. The symbol $T_i$ denotes the length of the plan $\nu_i$ and the trajectory $\sigma_i$. The trajectory of the multi-robot system $\mathcal{R} = \{r_1, \ldots, r_n\}$ is denoted by $\Sigma = [\sigma_1, \sigma_2, \ldots, \sigma_n]$, where $\sigma_i$ denotes the trajectory of robot $r_i$.
\end{definition}
        
\subsubsection{Optimality Criteria for a Trajectory}
The cost of executing a trajectory $\sigma_i = (\sigma_i(0), \sigma_i(1), \ldots, \sigma_i(T_i))$ is equal to its length $T_i$.
Now, the quality of a multi-robot trajectory $\Sigma$ is captured by one of the following two attributes. 

\begin{definition}[Makespan]
The makespan of the trajectories $\Sigma = [\sigma_1, \sigma_2, \ldots, \sigma_n]$ is given by
        $C = \max\limits_i T_i$.
\end{definition}

\begin{definition}[Total cost]
The total cost of the trajectories $\Sigma = [\sigma_1, \sigma_2, \ldots, \sigma_n]$ is given by
        $C = \sum\limits_i T_i$.
\end{definition}
   
Note that the makespan and total cost are equal for a single robot system. We will use the terms \emph{plan} and \emph{trajectory} interchangeably to denote the solution from our algorithm.

\subsection{Problem Definition}
\label{sec-prob}

Here, we provide the formal definition of the problem.

\noindent
\begin{definition}[Problem]
Given a workspace $\mathcal{W}$, a set of tasks $\mathcal{T}$, and a set of robots $\mathcal{R}$, find optimal makespan or optimal total cost collision-free trajectories $\Sigma$ for the robots such that all tasks are completed while also ensuring that the robots return to their initial positions.
\end{definition}


\begin{figure}[t]
\centering
{
    \label{fig:example1}
    \includegraphics[height=3cm, width=3cm]{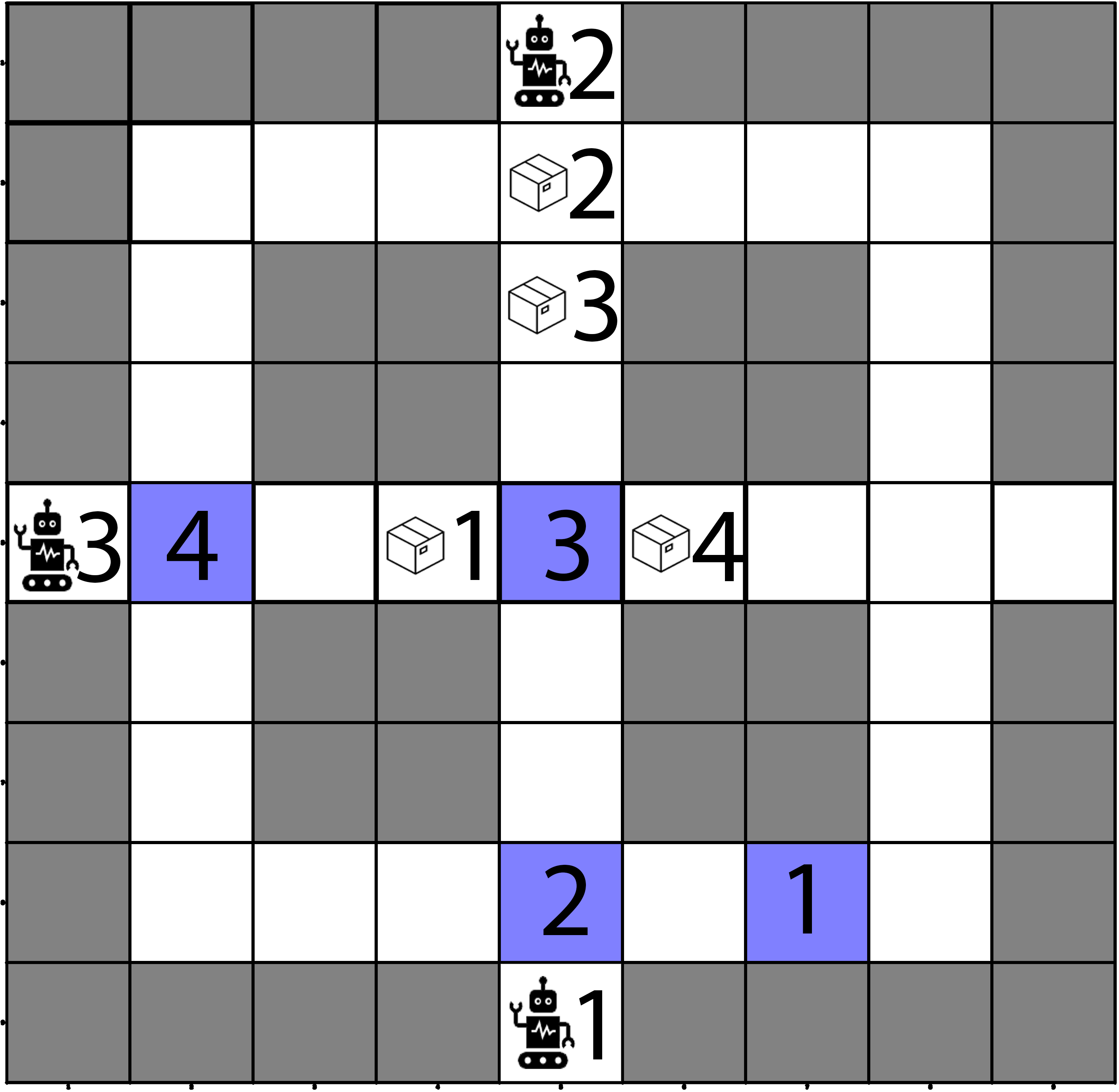}
} 
\qquad
{
    \label{fig:example2}
    \includegraphics[height=3cm, width=3cm]{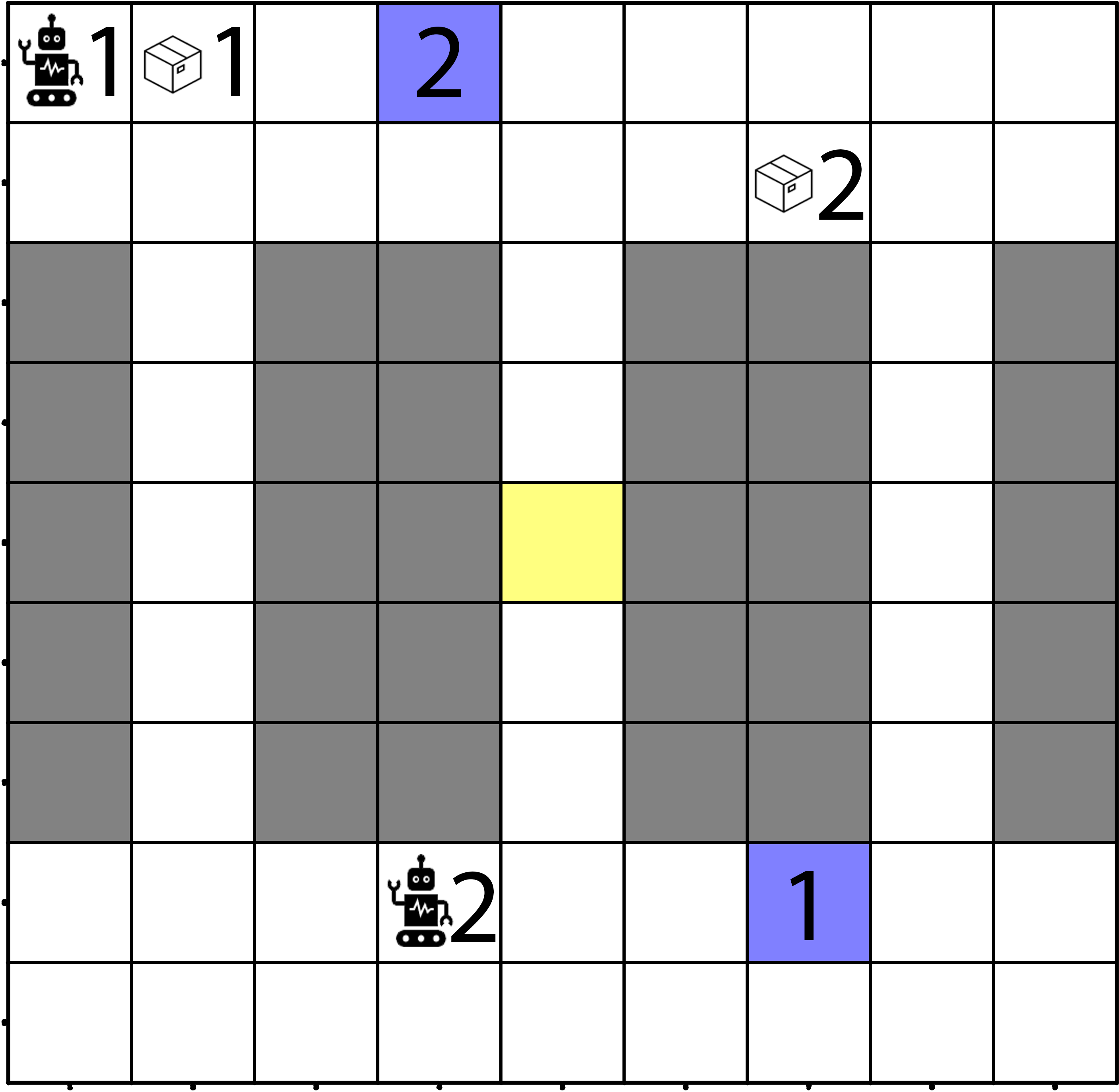}
}
\caption{Examples of workspaces showing warehouse scenarios a) without Intermediate Location, b) with an Intermediate Location.}
\label{fig:example}
\label{figurelabel}
\end{figure}  
   
   \noindent
   \textbf{Example.}
   Consider the workspaces shown in Figure~\ref{fig:example}. They represent typical warehouse scenarios. 
   Boxes in the images denote the pickup locations for these objects. The blue grid locations denote their drop locations. The grey-coloured grid blocks are occupied by obstacles and must be avoided. In the figure, the robots are shown in their initial locations. The robots can carry multiple objects at a time. To pick up an object, a robot needs to be in the grid block where the object is placed. The same is true for dropping an object. The problem is to find the task assignment to the robots to decide which robot should carry which object to its goal location and the collision-free trajectories for the robots to carry out their tasks successfully.
   
In Figure~\ref{fig:example}a, there are $4$ objects that need to be moved to some specific goal locations. Three robots $r_1$, $r_2$ and $r_3$ have 
to move the four objects from their current locations to their goal locations.
In Figure~\ref{fig:example}b, the yellow block denotes the intermediate drop block. A robot can drop an object on the yellow block, and the object can be picked up from there by another robot. Thus, having an intermediate block allows the robots to collaborate on delivering a specific object. 
In the scenario presented in Figure~\ref{fig:example2}, let us attempt to find the  plan with optimal makespan. The collision-free plan without the intermediate drop would be $r_1$ completing task $t_2$ and returning to its base location in 16 steps and $r_2$ completing $t_1$ and returning to its base location in 26 steps. So the makespan of this plan becomes 26.
\longversion{
The collision-free trajectory for the two robots $r_1$ and $r_2$ are shown in Figure~\ref{fig:traj1}.
\begin{figure}[t]
\centering
\begin{tabular}{cll}
\toprule
time & \ \ \ \ \ \ \ \ $r_1$ & \ \ \ \ \ \ \ \ $r_2$\\
\midrule
    0 & ($Start$, (0, 0)) & ($Start$, (7, 3)) \\
    1 & ($Move$, (1, 0)) &($Move$, (7, 2)) \\
    2 & ($Move$, (1, 1)) & ($Move$, (7, 1)) \\
    3 & ($Move$, (1, 2)) & ($Move$, (6, 1)) \\
    4 & ($Move$, (1, 3)) & ($Move$, (5, 1)) \\
    5 & ($Move$, (1, 4)) & ($Move$, (4, 1)) \\
    6 & ($Move$, (1, 5)) & ($Move$, (3, 1)) \\
    7 & ($Move$, (1, 6)) & ($Move$, (2, 1)) \\ 
    8 & ($Pick_2$, (1, 6)) & ($Move$, (1, 1)) \\
    9 & ($Move$, (0, 6)) & ($Move$, (0, 1)) \\
    10 & ($Move$, (0, 5)) & ($Pick_1$, (0, 1)) \\
    11 & ($Move$, (0, 4)) & ($Move$, (1, 1)) \\
    12 & ($Move$, (0, 3)) & ($Move$, (2, 1)) \\
    13 & ($Drop_2$, (0, 3)) & ($Move$, (3, 1)) \\
    14 & ($Move$, (0, 2)) & ($Move$, (4, 1)) \\
    15 & ($Move$, (0, 1)) & ($Move$, (5, 1)) \\
    16 & ($Return$, (0, 0)) & ($Move$, (6, 1)) \\
    17 & ($---$, (0, 0)) & ($Move$, (7, 1)) \\
    18 & ($---$, (0, 0)) & ($Move$, (7, 2)) \\
    19 & ($---$, (0, 0)) & ($Move$, (7, 3)) \\
    20 & ($---$, (0, 0)) & ($Move$, (7, 4)) \\ 
    21 & ($---$, (0, 0)) & ($Move$, (7, 5)) \\
    22 & ($---$, (0, 0)) & ($Move$, (7, 6)) \\
    23 & ($---$, (0, 0)) & ($Drop_1$, (7, 6)) \\
    24 & ($---$, (0, 0)) & ($Move$, (7, 5)) \\
    25 & ($---$, (0, 0)) & ($Move$, (7, 4)) \\ 
    26 & ($---$, (0, 0)) & ($Return$, (7, 3)) \\
    \bottomrule
\end{tabular}
\caption{Trajectories of the two robots for the problem shown in Figure~\ref{fig:example}(a)}
\label{fig:traj1}
\end{figure}
}
If we allow the robots to use the intermediate block for object transfer, $r_1$ can pick up $t_1$ and drop it to the intermediate block; then it can continue to pick and drop $t_2$ and return to its base location in 24 steps. But this reduces the time taken by $r_2$ to process $t_1$. Now, $r_2$ can pick up $t_1$ from the intermediate location, drop it to its drop location, and come back to its base station. Execution of this plan takes 21 steps to complete, thus making the overall makespan 24.  
\longversion{Since we optimize the makespan, the total cost metric may increase. In this scenario, the total cost increases from 42 to 45. The trajectories for both of the robots are shown in Figure~\ref{fig:traj2}.
\begin{figure}[t]
\centering
    \begin{tabular}{cll}
    \toprule
    time & \ \ \ \ \ \ \ \ $r_1$ & \ \ \ \ \ \ \ \ $r_2$\\
    \midrule
    0 & ($Start$, (0, 0)) & ($Start$, (7, 3)) \\
    1 & ($Move$, (0, 1)) & ($Move$, (7, 4)) \\
    2 & ($Pick_1$, (0, 1)) & ($Move$, (6, 4)) \\
    3 & ($Move$, (0, 2)) & ($Move$, (5, 4)) \\
    4 & ($Move$, (1, 2)) & ($Move$, (4, 4)) \\
    5 & ($Move$, (1, 3)) & ($Move$, (4, 4)) \\
    6 & ($Move$, (1, 4)) & ($Move$, (4, 4)) \\
    7 & ($Move$, (2, 4)) &  ($Move$, (4, 4)) \\
    8 & ($Move$, (3, 4)) & ($Move$, (4, 4)) \\
    9 & ($Move$, (4, 4)) & ($Move$, (5, 4)) \\
    10 &  ($InterDrop_1$, (4, 4)) & ($Move$, (5, 4)) \\
    11 &  ($Move$, (3, 4)) &($Move$, (4, 4)) \\
    12 &  ($Move$, (2, 4)) & ($InterPick_1$, (4, 4)) \\
    13 & ($Move$, (1, 4)) & ($Move$, (5, 4)) \\
    14 & ($Move$, (1, 5)) & ($Move$, (6, 4)) \\
    15 & ($Move$, (1, 6)) & ($Move$, (7, 4)) \\
    16 & ($Pick_2$, (1, 6)) & ($Move$, (7, 5)) \\
    17 & ($Move$, (0, 6)) & ($Move$, (7, 6)) \\
    18 & ($Move$, (0, 5)) & ($Drop_1$, (7, 6)) \\
    19 & ($Move$, (0, 4)) & ($Move$, (7, 5)) \\
    20 & ($Move$, (0, 3)) & ($Move$, (7, 4)) \\
    21 & ($Drop_2$, (0, 3)) & ($Return$, (7, 3)) \\
    22 & ($Move$, (0, 2)) & ($---$, (7, 3)) \\
    23 & ($Move$, (0, 1)) & ($---$, (7, 3)) \\
    24 & ($Return$, (0, 0)) & ($---$, (7, 3)) \\
    \bottomrule
\end{tabular}
\caption{Trajectories of the two robots for the problem shown in Figure~\ref{fig:example}(b)}
\label{fig:traj2}
\end{figure}

}
\shortversion{The collision-free trajectories for the two robots with and without intermediate location are provided in Section~\ref{sec-prob} of~\cite{optITMPjournal}. }

Thus, intermediate locations help in finding a better plan for our optimization criteria, and our goal would be to design a planner that can efficiently exploit the availability of such opportunities.
\section{Integrated Task and Path Planning Algorithm}
\label{sec-algorithm}

In this section, we provide an algorithm to solve the problem described in Section~\ref{sec-problem}. One could reduce the problem to an Integer-Linear Programming or an SMT-solving problem and generate a solution for the task assignment as well as the trajectories for the robots. However, this monolithic approach rarely scales up with the number of robots, the number of tasks, and the size of the workspace. Instead, we embrace a decoupled approach where the task and the path planning problems are solved independently. However, through an interaction between the task and the path planner, we ensure that the finally generated plans satisfy the task completion requirement and that the corresponding paths are collision-free and optimal.

Our proposed methodology is shown in Algorithm~\ref{alg:Algorithm1}.
We advocate the use of an SMT solver to solve complex task assignment problems. The procedure \textsc{task\_planner} takes  $\mathcal{W}$, $\mathcal{R}$, $\mathcal{T}$, a set $\mathcal{A}$ of forbidden task assignments, a lower bound $l\_b$, and an upper bound $u\_b$ as inputs. It produces as output a task assignment $\mathcal{L} = [\mathcal{L}_1, \mathcal{L}_2, \ldots, \mathcal{L}_{|\mathcal{R}|}]$, with minimum total cost or makespan within specified bounds. It returns $\emptyset$ if there does not exist a feasible task assignment within the bounds. Here, $\mathcal{L}_i$ denotes the sequence of locations that robot $r_i$ must visit. In Section~\ref{sec-taskplanning}, we will present the details of the task planner with an example of a warehouse pick-and-drop application.

The following procedure \textsc{path\_planner} takes the task assignment $\mathcal{L}$ produced by the \textsc{task\_planner} procedure and generates \emph{optimal} and \emph{collision-free} trajectories. The procedure also returns the trajectory's total cost or makespan depending upon the optimization criterion. In Section~\ref{sec-motionplanning}, we will present the details of the path planner.

The main procedure \textsc{integrated\_planner} induces an interaction between the task planner and the path planner to find the optimal collision-free trajectories for the robots. 
There could be several task assignments with the same cost. Thus, once a task assignment $\mathcal{L}$ is produced by the task planner, we need to ensure that the task planner does not generate the same task assignment again. We use the set $\mathcal{A}$ for this purpose. We keep on storing the task assignments with the same cost in $\mathcal{A}$ and provide it as the set of prohibited assignments while invoking the task planner with the same lower bound of the cost.
This set is initialized as an empty set.
We initialize $cur\_task\_cost$ (denoting the cost of the current task assignment) as 0 and $opt\_plan\_cost$ (denoting the minimum cost of the collision-free paths for any assignment) as $\infty$ and repeat the procedure below until $cur\_task\_cost$ becomes equal to $opt\_plan\_cost$. We invoke the \textsc{task\_planner} with the $cur\_task\_cost$ as lower bound and $opt\_plan\_cost$ as upper bound to get the best task assignment $\mathcal{L}$ with cost $task\_cost$ based on some heuristic cost of movements between important locations. If the task planner cannot produce a plan (returns $\emptyset$), we terminate the loop. Otherwise, for this task assignment $\mathcal{L}$, we invoke the \textsc{path\_planner}, which outputs the collision-free trajectory with cost $plan\_cost$. If we find that the new task assignment $\mathcal{L}$ has a higher cost compared to $cur\_task\_cost$, then we update $cur\_task\_cost$ with $task\_cost$ and reset the exclusion's list $\mathcal{A}$. We add this task assignment $\mathcal{L}$ to the $\mathcal{A}$. We update the $opt\_plan$ and $opt\_plan\_cost$ if the current trajectory has a better cost.

\begin{algorithm}[t]
\caption{Integrated Planner using Task and Path Planner}
\label{alg:Algorithm1}
\begin{algorithmic}[1]
\Procedure {task\_planner }{$\mathcal{W}$, $\mathcal{R}$, $\mathcal{T}$, $\mathcal{A}$, $l\_b$, $u\_b$}
\State// find optimal task assignments using a heuristic cost for movements.
\State $\mathtt{return}\ \ \langle \mathcal{L}, task\_cost \rangle$
\EndProcedure
\medskip
\Procedure {path\_planner }{$\mathcal{W}$, $\mathcal{R}$, $\mathcal{L}$}
\State // find the optimal collision-free trajectories for robots following the given task assignments in L.
\State $\mathtt{return} \ \langle plan, plan\_cost \rangle$
\EndProcedure

\medskip
\Procedure {integrated\_planner }{$\mathcal{W}$, $\mathcal{R}$, $\mathcal{T}$}
\State $cur\_task\_cost \gets 0$; $opt\_plan\_cost \gets \infty$
\State $opt\_plan \gets \emptyset$; $\mathcal{A} \gets \emptyset$
\While{$cur\_task\_cost < opt\_plan\_cost$}

\State $\langle \mathcal{L}, task\_cost \rangle \gets \textsc{task\_planner}$ ($\mathcal{W}$, $\mathcal{R}$, $\mathcal{T}$, $\mathcal{A}$, $cur\_task\_cost$, $opt\_plan\_cost$)
\If {$\mathcal{L}$ == $\emptyset$}
\State {$break$}
\EndIf
\State $\langle plan, plan\_cost \rangle \gets \textsc{path\_planner}$ ($\mathcal{W}$, $\mathcal{R}$, $\mathcal{L}$)

\If {($cur\_task\_cost < task\_cost$)}
\State $cur\_task\_cost \gets task\_cost$
\State $\mathcal{A} \gets \emptyset$
\EndIf
\State $\mathcal{A} \gets \mathcal{A} \cup \{\mathcal{L}\}$

\If {($plan\_cost < opt\_plan\_cost$)}
\State $opt\_plan \gets plan$
\State $opt\_plan\_cost \gets plan\_cost$
\EndIf

\EndWhile
\State $\mathtt{return}$ $\langle opt\_plan, opt\_plan\_cost \rangle$ 
\EndProcedure
\end{algorithmic}
\end{algorithm}

Now, we formally prove that Algorithm~\ref{alg:Algorithm1} produces the optimal trajectories satisfying the task requirements.

\begin{theorem}[Optimality]
\label{th:th1}
There does not exist a task assignment for which the  cost of the collision-free trajectories would be less than the cost of the trajectories returned by Algorithm~\ref{alg:Algorithm1}. 
\end{theorem}
\begin{proof} 
Let us assume that Algorithm~\ref{alg:Algorithm1} returns collision-free trajectories for the robots with cost $C$ for a task assignment $\mathcal{L}$. The heuristic cost for the assignment is $C_h$. Now, let us assume that there exists a task assignment $\mathcal{L'}$ for which the cost of the collision-free trajectories is $C'$ where $C'< C$, but this task assignment was not considered by Algorithm~\ref{alg:Algorithm1}. The heuristic cost for the assignment $\mathcal{L'}$ is $C'_h$. As heuristic cost must always be a lower bound for the cost of the collision-free trajectories, $C_h \leq C$ and $C_h' \leq C'$. Then either (I)~$C_h' < C_h$ or (II)~$C_h \leq C_h'$.

\textit{Case I: }
In this case, $\mathcal{L'}$ must have been considered by the planner before $\mathcal{L}$ as the task planner returns the task assignment with the minimum possible heuristic cost. 

\textit{Case II: }
As $C_h' \leq C'$ and $C'< C$, therefore $C_h' < C$.
In this case, the planner must have considered $\mathcal{L'}$ after generating collision-free trajectories for $\mathcal{L}$ as $C_h' < C$ and $C_h \leq C_h'$. Our Integrated Planner explores all task assignments with heuristic costs less than $C$.

Thus, in both cases, our assumption that Algorithm~\ref{alg:Algorithm1} did not consider $\mathcal{L}'$ is wrong.
Hence, if the heuristic cost considered in the task planning procedure gives a lower bound on cost and the Path  Planner gives the minimum cost collision-free paths corresponding to the task assignment, then the integrated planner will always generate collision-free trajectories for the robots with optimal cost. 
\end{proof}

\noindent
\textbf{Note:} As the number of task assignments is finite for a well-formed MAPD instance, the optimality of Algorithm~\ref{alg:Algorithm1} establishes its \emph{completeness} as well.

\smallskip
\noindent
\textbf{Example.}  We illustrate the algorithm on the example introduced in Figure~\ref{fig:example}a in Section~\ref{sec-problem} with makespan as optimization criteria. Here, we use A* search algorithm~\cite{HartNR68} to find a trajectory for a robot between two locations. In the below task assignments, pickup represents move and pickup. Similarly, the drop represents move and drop. Since there is no intermediate location, all pickups are the boxes' initial locations, and drops are their respective drop locations. The minimum makespan returned by the task planner is 18, and the corresponding task assignment is as follows: 

\smallskip
\begin{tabular}{ll}
$r_1$ :& $\mathtt{pickup-1}$, $\mathtt{drop-1}$ \\
$r_2$ :& $\mathtt{pickup-2}$, $\mathtt{pickup-3}$, $\mathtt{drop-3}$, $\mathtt{drop-2}$ \\
$r_3$ :& $\mathtt{pickup-4}$, $\mathtt{drop-4}$\\
\end{tabular}
\smallskip


In the above task assignment, $r_1$ starts from grid location (8, 4), visits the grid location (4, 3) to pick up object-1 and then visits grid location (7, 6) to drop object-1 and then finally return to grid location (8, 4). The distances computed by the A* algorithm for these movements are 5, 6, and 3, respectively. Also, $r_1$ spends two units of time step to pick and drop the object, thus making the total time steps 16. Similarly, the cost for robots $r_2$ and $r_3$ are 18 and 12, respectively. Therefore, the effective makespan of the plan is 18. 
This heuristic cost is generated by calculating the costs individually without considering the robot-robot collisions. Using the task assignment, we compute collision-free trajectory using the path planner. The cost of collision-free trajectories the path planner returns is 19, 18, and 17, respectively. So, the overall makespan becomes 19. 
Since the estimated task assignment cost is 18 and the collision-free cost is 19, there may be some plans with a cost of 18, resulting in a makespan less than 19. So, we continue to find more plans and obtain the next task assignment as follows:

\smallskip
\begin{tabular}{ll}
$r_1$ :& $\mathtt{pickup-1}$, $\mathtt{drop-1}$ \\
$r_2$ :& $\mathtt{pickup-2}$, $\mathtt{drop-2}$ \\
$r_3$ :& $\mathtt{pickup-4}$, $\mathtt{pickup-3}$ $\mathtt{drop-3}$ $\mathtt{drop-4}$\\
\end{tabular}
\smallskip

The makespan of the above task assignment is 18.
The path planner returns a plan with a makespan of 19, the same as the previously found plan's makespan.
We continue searching for task assignments. The third assignment that we obtain also has a makespan of 18. It is as follows: 

\smallskip
\begin{tabular}{ll}
$r_1$ :& $\mathtt{pickup-1}$, $\mathtt{drop-1}$ \\
$r_2$ :& $\mathtt{pickup-2}$, $\mathtt{pickup-3}$, $\mathtt{drop-2}$, $\mathtt{drop-3}$ \\
$r_3$ :& $\mathtt{pickup-4}$, $\mathtt{drop-4}$\\
\end{tabular}

\smallskip
The above task assignment differs slightly from the first assignment, in which $r_2$ drops object-2 before dropping object-3. The estimated cost returned by the task planner for $r_1$, $r_2$, and $r_3$ is 16, 18, and 12, respectively. Executing the path planner with this task assignment returns a collision-free trajectory with costs of 18, 18, and 12, respectively, thus making the makespan 18. So, this collision-free trajectory becomes the minimum collision-free trajectory, and the minimum cost is updated to 18. As the collision-free cost is not greater than the estimated cost, we terminate the algorithm.
\section{Applications to Warehouse Management}
\label{sec-taskplanning}

In this section, we illustrate our planning mechanism for the object pick-and-drop application in a warehouse scenario, as shown in Figure~\ref{fig:example}. As the tasks are pick-and-drop, $L_i$ for each task $t_i$ contains two entries: $L_i(0)$ denotes the pickup location and $L_i(1)$ represents the drop location.


\subsection{Task Planning Algorithm}
\label{sec-taskconstraints}

\begin{algorithm}[t]
\caption{Task Planner}
\label{alg:Algorithm2}
\begin{algorithmic}[1]
\Procedure {task\_planner }{$\mathcal{W}$, $\mathcal{R}$, $\mathcal{T}$, $\mathcal{A}$, $l\_b$, $u\_b$}
\State $\mathcal{S} \gets \mathtt{generate\_smt\_instance}$ ($\mathcal{W}$, $\mathcal{R}$, $\mathcal{T}$, $\mathcal{A}$)
\If {$\mathcal{S}$.$\mathtt{check} ( )$ $\neq$ SAT}
        \State $\mathtt{return}  \ \emptyset$
\EndIf
\While{$(l\_b \leq u\_b)$}
    \State $\mathcal{S'} \gets  \mathcal{S}$
    \State $mid \gets (l\_b+u\_b)/2$ 
    \State $\mathcal{S} \gets \mathcal{S} \wedge (cost \ge l\_b) $
    \State $\mathcal{S} \gets \mathcal{S} \wedge (cost \le mid)$
    \If {$\mathcal{S}.\mathtt{check}()$ = SAT}
        \State $u\_b \gets \mathcal{S}.\mathtt{get}(cost) - 1$
    \Else
        \State $l\_b \gets mid + 1$
    \EndIf
    \State $\mathcal{S} \gets  \mathcal{S'}$
\EndWhile
\State $\mathcal{L} \gets \mathcal{S}.\mathtt{get} (task\_assignment)$
\State $cost \gets \mathcal{S}.\mathtt{get}  (cost)$
\State $\mathtt{return} \ \langle \mathcal{L}, cost\rangle$
\EndProcedure
\end{algorithmic}
\end{algorithm}

The overall SMT-based task planning algorithm is shown in Algorithm~\ref{alg:Algorithm2}. The $\mathtt{generate\_smt\_instance}$ function generates the SMT constraints for the task planner. We use the notion of \emph{action-step} in our SMT formulation. In each action step, all the robots can perform an action related to movement, pickup, or drop. In our constraints, we keep track of the time taken for each action step for each robot. There is no constraint on how long these actions can take here; we do not generate the final paths but rather just the task assignment. The time required for an action that requires a movement from location $\bf{x}$ to location $\bf{x'}$ is captured by $\mathtt{dist}(\bf{x},\bf{x'})$, as we assume a movement from one grid cell to another takes one unit of time. 
We compute $\mathtt{dist}(\bf{x},\bf{x'})$ using the $\textsf{A*}$ search algorithm~\cite{HartNR68}, which is guaranteed to be an under-approximation of the distance between $\bf{x}$ and $\bf{x'}$ while computing the collision-free trajectories for the robots. 
For a task assignment problem, the number of action steps is denoted by $Z$, which is the same for all the robots. 

\shortversion{
We now describe the constraints to capture the pick-and-drop problem. We define $LOC$ as a set of all the task's pickup and drop locations. Thus, $LOC = \bigcup\limits_{t_m\in\mathcal{T}} \{L_m(0), L_m(1)\}$.
We define the following decision variables:
(a) \mbox{$pos_{i,j} \in \{LOC \cup s_i\}$}, denoting the location of robot $r_i$ after the $j^{th}$ action-step,
(b) $pos\_time_{i,j}$, denoting the time step at which robot $r_i$ is at $pos_{i,j}$ location in the $j^{th}$ action-step,
(c) $action_{i,j}$ denoting the task's object $t_i$ on which robot $r_i$ will perform an action in the $j^{th}$ action-step (it is $-1$ if no action is performed by the robot on any task),
(d) $loc_{i,j}$, denoting the location of task $t_i$ in the $j^{th}$ action-step, which can be either $L_i(0)$, $L_i(1)$ or an intermediate location (it is $-1$ in case the task object is being carried by some robot),
(e) $loc\_time_{i,j}$, denoting the time at which task $t_i$ will be available at $loc_{i,j}$ at the $j^{th}$ action-step (it is $-1$ if the task object is in transition), and
(f) $being\_carried_{i,j}$, denoting the identifier of the robot carrying the object of task $t_i$ in the $j^{th}$ action-step (it is $-1$ when not carried by any robot).

The following constraints capture the initial state $\textsc{Z = 0}$ of the system.
\begin{subequations}
\begin{align}
    \forall r_i\in\mathcal{R}, & \ pos_{i,0} = {s}_i \wedge \ pos\_time_{i,0} = 0 \ \wedge \ action_{i,0} = -1  \nonumber\\
    \forall t_i\in\mathcal{T}, & \ loc_{i,0} = L_i(0) \ \wedge \ being\_carried_{i,0} = -1 \nonumber
\end{align}
\end{subequations}
We formulate constraints for the following actions listed below. At each action step $j$, the robot $r_i$ can perform one of the following actions:
(a) $stay(r_i,j)$: The robot $r_i$ can stay at the same location,
(b) $return(r_i, j)$: The robot $r_i$ can return to its initial location.
(c) $pick(r_i,t_m,j)$: The robot $r_i$ can move and pickup task $t_m$'s object from its initial position,
(d) $drop (r_i,t_m,j)$: The robot $r_i$ can move and drop task $t_m$'s object to its destination,
(e) $pick\_intermediate(r_i,t_m,i_n,j)$: The robot $r_i$ can move and pickup task $t_m$'s object from the intermediate location $i_n$ immediately. This action is performed when the task $t_m$'s object already exists at $i_n$,
(f) $wait\_intermediate(r_i,t_m,i_n,j)$: This action is the same as the previous action, but it is performed when robot $r_i$ will pickup task $t_m$'s object after waiting for the  object to be available at $i_n$, 
(g) $drop\_intermediate(r_i,t_m,i_n,j)$: The robot $r_i$ can move and drop task $t_m$'s object to the intermediate location $i_n$.

The final set of constraints is obtained as a conjunction of the mentioned constraints as described below: 
\begin{subequations}
\begin{align}
    \bigwedge_{r_i \in \mathcal{R}} &\bigwedge_{j=1}^{Z}\ \bigg( stay(r_i,j)  \ \vee  
    \bigvee_{k \in\{{s_i}\}\cup LOC} \Big((pos_{i,j-1} = k) \ \wedge \nonumber \\
    &\ \ \ \big( return(r_i, j) \ \vee \nonumber \\
    &\ \ \ \bigvee_{t_m \in \mathcal{T}} \big(pick(r_i,t_m,j) \ \vee drop (r_i,t_m,j) \ \vee  \nonumber \\
    &\ \ \  \bigvee_{i_n\in\mathcal{I}}(pick\_intermediate(r_i,t_m,i_n,j) \ \vee \nonumber \\ 
    &\ \ \ \ \ \ \ \ \ \   wait\_intermediate(r_i,t_m,i_n,j) \ \vee \nonumber \\
    &\ \ \ \ \ \ \ \ \ \ drop\_intermediate(r_i,t_m,i_n,j)\ )\  \nonumber \big)\Big)\bigg)  
\label{eq:combinedint_main}
\end{align}
\end{subequations}
The equation represents all possible actions the robot $r_i$ can take at $j^{th}$ action step. 
The complete SMT instance also consists of other constraints, such as the exclusion of already found task assignments $\mathcal{A}$ and operational constraints like weights and deadlines.  
The complete details of the SMT Formulation with each of the actions and the constraints are presented in  Section~\ref{sec-taskconstraints} in~\cite{optITMPjournal}.

Note that, the non-movement action like $pick$, $drop$, $pick\_intermediate$ have an implicit move action along with them as $move\_pick$, $move\_drop$, $move\_pick\_intermediate$. The value of $Z$ controls the number of actions one robot must perform. Each task plan requires a mandatory action step for $return\_to\_base$ action. For every task, a robot requires at least two action steps for $pick$ and $drop$. For example, if two robots have to perform four tasks, with $Z$ = 5, each robot can be assigned two tasks. One possible task plan for each of the robots with five action steps can be $move\_pick$, $move\_drop$, $move\_pick$, $move\_drop$, and $return\_to\_base$. If $Z$ is increased to 7 for the previous problem, then one of the robots may perform $3$ tasks, and the other can perform $1$ task.
}

\longversion{
\subsection{SMT Encodings Of Constraints}
\label{appendix_smt_encodings}
In this section, we describe the constraints in detail to capture two variants of the pick-and-drop problem.

\subsubsection{Completing pick-and-drop tasks}
Here, we present the SMT constraints to capture the basic object pick-and-drop problem as illustrated in Figure~\ref{fig:example}.
We define $LOC$ as a set of all the task's pickup and drop locations. Thus, $LOC = \bigcup\limits_{t_m\in\mathcal{T}} \{L_m(0), L_m(1)\}$.

The following are the variables used to track the state of the system.
\begin{itemize}
    \item $pos_{i,j}$ denotes the location of robot $r_i$ after the $j^{th}$ action-step. This location can be one of the locations from the sets $LOC$ and $s_i$ for all $j\geq1$.
    \item $pos\_time_{i,j}$ denotes the time step at which robot $r_i$ is at location $pos_{i,j}$ in the $j^{th}$ action-step.
    \item $action_{i,j}$ denotes on which task's object $r_i$ will perform action in the $j^{th}$ action-step. The value of the variable can be either $-1$ if no action is performed or the task number.
    \item $loc_{i,j}$ denotes the location of task $t_i$ in the $j^{th}$ action-step. This location can be either $L_i(0)$ or $L_i(1)$, or it can be $-1$ in case the task object is being carried by some robot. 
    \item  $being\_carried_{i,j}$ denotes by which robot the object of task $t_i$ is being carried in the $j^{th}$ action-step. It is either the identifier of the robot if the task is in transition or $-1$ if it is steady.
\end{itemize}

The initial state of the system is captured by the following constraints.
\begin{align}
    \forall r_i\in\mathcal{R}, & \ pos_{i,0} = {s}_i \wedge \ pos\_time_{i,0} = 0 \ \wedge \ action_{i,0} = -1  \nonumber\\
    \forall t_i\in\mathcal{T}, & \ loc_{i,0} = L_i(0) \ \wedge \ being\_carried_{i,0} = -1
\label{eq:initConstraints}
\end{align}
A robot can go to $L_m(0)$ only if it picks up the object of task $t_m$ from there.
If robot $r_i$ wants to pick up an object from one of the pickup locations in action step $j$, then the constraints formulation is as mentioned below.
\begin{subequations}
    \begin{align}
        pick &(r_i,t_m,j) \equiv\nonumber \\
               & loc_{m,j-1} = L_m(0) \\ 
        \wedge \ & pos_{i,j} = L_m(0)\ \wedge being\_carried_{m,j} = i \\ 
        \wedge \ & pos\_time_{i,j} = pos\_time_{i,j-1} + \nonumber \\
            & \qquad \qquad  dist(pos_{i,j-1},L_m(0)) + 1 \\
        \wedge \ & loc_{m,j} = -1\ \wedge action_{i,j} = m 
    \end{align}
\label{eq:pick}
\end{subequations}
Equation~\ref{eq:pick}(a) captures that task $t_m$ is at location $L_m(0)$ in the $j-1$ action-step. Equation~\ref{eq:pick}(b) captures that robot $r_i$ is at location $L_m(0)$ in action-step $j$ and the object for task $t_m$ is being carried by robot $r_i$ in action-step $j$. Equation~\ref{eq:pick}(c) captures the time taken by robot $r_i$ while moving from its location in the previous action-step $pos_{i,j-1}$ to its location in the current action-step $L_m(0)$ {and one unit of time for picking up the task $t_m$ by $r_i$. Equation~\ref{eq:pick}(d) ensures that $loc_{m,j}$ is set to $-1$ as an object for task $t_m$ is being carried by a robot now and sets $action_{i,j}$ as $m$ to indicate pickup of the object $t_m$ by robot $r_i$ in action-step $j$.

Similarly, a robot can go to one of the drop locations only if it drops an object there. If $r_i$ wants to drop an object to one of the drop locations in action-step $j$, then the constraints formulation is as below.
\begin{subequations}
    \begin{align}
        drop & (r_i,t_m,j) \equiv\nonumber \\
               & being\_carried_{m,j-1} = i \\
        \wedge \ & pos_{i,j} = L_m(1)\ \wedge being\_carried_{m,j} = -1 \\
        \wedge \ & pos\_time_{i,j} =  pos\_time_{i,j-1} + \nonumber\\
        \qquad & \qquad \qquad \mathtt{dist}(pos_{i,j-1},L_m(1)) + 1 \qquad \\
        \wedge \ & loc_{m,j} = L_m(1)\ \wedge action_{i,j} = m  
    \end{align}
\label{eq:drop}
\end{subequations}
Equation~\ref{eq:drop}(a) captures that task $t_m$ must be carried by robot $r_i$ in action-step $j-1$ to be able to drop it in action-step $j$. Equation~\ref{eq:drop}(b) captures that robot $r_i$ is at location $L_m(1)$ in action-step $j$ and changes $being\_carried_{m,j}$ to $-1$ as the object will be dropped. Equation~\ref{eq:drop}(c) captures the time taken by robot $r_i$ while moving from its location in the previous action-step $pos_{i,j-1}$ to its location in the current action-step $L_m(1)$ and one unit of time to drop the task $t_m$ by $r_i$. Equation~\ref{eq:drop}(d) set $loc_{m,j}$ to indicate that the object for task $t_m$ has been dropped at its final location in action-step $j$ and $action_{i,j}$ to $m$ to indicate dropping of the object for task $t_m$ by robot $r_i$ in action-step $j$. 

A robot can also do nothing for one action step, which is captured as follows.
\begin{align}
stay (r_i,j) \equiv \
            &  pos_{i,j} = pos_{i,j-1} \ \wedge \ action_{i,j} = -1 \nonumber \\
    \wedge \  &  pos\_time_{i,j} = pos\_time_{i,j-1}
\label{eq:stay}
\end{align}

A robot can also return to the base station from a drop location if it is no longer required to do more tasks.
\begin{subequations}
    \begin{align}
        return & (r_i, j) \equiv\nonumber \\
        & pos_{i, j} = s_i \ \wedge \ action_{i,j} = -1 \ \wedge \\ 
        & pos\_time_{i,j} =  pos\_time_{i,j-1} + \mathtt{dist}(pos_{i,j-1},s_i) 
    \end{align}
\label{eq:return}
\end{subequations}
Equation~\ref{eq:return}(a) captures that the robot $r_i$ is at base station $s_i$ at action-step j. Equation~\ref{eq:return}(b) captures the time taken by robot $r_i$ while moving from its location in the previous action-step $pos_{i,j-1}$ to its base station in the current action-step.

Combining Equations~\eqref{eq:pick} - \eqref{eq:return}, for each robot $r_i$ for each possible action-step $j$, we get the constraint below:
\begin{align}
    &\bigwedge_{r_i \in \mathcal{R}} \bigwedge_{j=1}^{Z}\ \bigg( stay (r_i,j)  \ \vee  
    \bigvee_{k \in\{{s_i}\}\cup LOC} \Big( (pos_{i,j-1} = k) \ \wedge \nonumber \\
    &   return(r_i, j)   \bigvee_{t_m \in \mathcal{T}}\big( pick (r_i,t_m,j) \ \vee drop (r_i,t_m,j)\ \big)\Big)\bigg) 
\label{eq:combined}
\end{align}

We now add the constraints to enforce that the task objects move only when being carried by one of the robots.
\begin{align}
    \bigwedge_{t_m \in \mathcal{T}} \bigwedge_{j=1}^{Z} & \big(\bigwedge_{r_i\in\mathcal{R}}action_{i,j}\neq m\big) \implies 
    (loc_{m,j} = loc_{m,j-1} \nonumber \\
    &\wedge being\_carried_{m,j} = being\_carried_{m,j-1})
\label{eq:consistency}
\end{align}
Equation~\ref{eq:consistency} ensures that if no robot is performing an action on task $t_m$, then $t_m$'s location and being carried status remain the same. Note that only picking up or dropping is classified as performing an action. A robot carrying a task's object does not mean that he is performing an action on that task.
\begin{align}
\ \  \bigwedge_{t_m \in \mathcal{T} } (loc_{m,Z} = L_m(1))  
\label{eq:endloc}
\end{align}

Equation~\eqref{eq:endloc} ensures that each task object is at its goal location in the last action step.

The final set of constraints is obtained as the conjunction of constraints capturing the initial states and those in Equations \eqref{eq:initConstraints}, \eqref{eq:combined}, \eqref{eq:consistency} and \eqref{eq:endloc}.

\subsubsection{Enabling collaboration}
In this subsection, we present the additional set of constraints that enables collaboration among the robots with the help of intermediate locations, as illustrated in Figure~\ref{fig:example}b.

A robot can visit one of the \emph{intermediate blocks} to either pick up or drop off an object. While picking up from an intermediate block, a validation of the timing consistency between the drop-off and pick-up of an object is required.
We introduce new SMT variables named $loc\_time_{i,j}$ to add this ability.
\begin{itemize}
    \item $loc\_time_{i,j}$ denotes the time step at which task $t_i$ will be available at $loc_{i,j}$ at the $j^{th}$ action-step. It is $-1$ if the task object is in transition.
\end{itemize}

Assume that a robot $r_1$ dropped the object of task $t_l$ at location $i_1$ in action step $j$ with $loc\_time_{l,j} = 20$. Now, suppose another robot $r_2$, which has been idle for all the action steps up to step $j+1$, goes to pick up this object. So, $pos_{2,j+1} = i_1$, but it is possible that $pos\_time_{2,j} + dist(pos_{2,j},i_1) < 20$. So even though $r_2$ will go to pick up the object at a later action step, it will reach the location before the task object is available there. Thus, in our constraints, we need to accommodate this possibility into the computation of $pos\_time$ as the action will be completed only when the pickup is done.

To accommodate the intermediate locations in $\mathcal{I}$  in our constraints we update $LOC$ as follows:
$$LOC = \bigg( \bigcup\limits_{t_m\in\mathcal{T}}\{L_m(0), L_m(1)\} \bigg) \cup \bigg(\bigcup\limits_{i_n\in\mathcal{I}}\{i_n\}\bigg)$$

Constraints formulation for $r_i$ picking up one of the task objects from one of the \emph{intermediate blocks} in action step $j$ is as below in Equation~\eqref{eq:pickint} and~\eqref{eq:waitint}.
\begin{subequations}
    \begin{align}
            pick&\_intermediate(r_i,t_m,i_n,j) \equiv \nonumber \\
           & loc_{m,j-1} = i_n \\ 
            \wedge \ & loc\_time_{m,j-1} \le pos\_time_{i,j-1} + \nonumber \\
            &\qquad \mathtt{dist}(pos_{i,j-1},i_n) + 1 \\
            \wedge \ & pos_{i,j} = i_n\ \wedge being\_carried_{m,j} = i \\  
            \wedge \ & pos\_time_{i,j} = pos\_time_{i,j-1} + \nonumber \\
            &\qquad dist(pos_{i,j-1},i_n) + 1 \\  
            \wedge \ & loc_{m,j} = -1\ \wedge loc\_time_{m,j} = -1 \\ 
            \wedge \ & action_{i,j} = m 
    \end{align}
\label{eq:pickint}
\end{subequations}

Equation~\eqref{eq:pickint} is similar to Equation~\eqref{eq:pick} except the extra constraint in Equation~\ref{eq:pickint}(b), which ensures that the task object is at the location before the robot reaches there to pick it up.
\begin{subequations}
\begin{align}
        wait&\_intermediate(r_i,t_m,i_n,j) \equiv \nonumber \\
        & loc_{m,j-1} = i_n \\ 
        \wedge\ & loc\_time_{m,j-1} > pos\_time_{i,j-1} + \nonumber \\
        &\qquad \mathtt{dist}(pos_{i,j-1},i_n) + 1 \\
        \wedge \ & pos_{i,j} = i_n\ \wedge being\_carried_{m,j} = i \\
        \wedge \ & pos\_time_{i,j} = loc\_time_{m,j-1} + 2 \\ 
        \wedge \ & loc_{m,j} = -1\ \wedge loc\_time_{m,j} = -1 \\ 
        \wedge \ & action_{i,j} = m 
\end{align}  
\label{eq:waitint}
\end{subequations}

Equation~\eqref{eq:waitint} is similar to  Equation~\eqref{eq:pick} except the changes in Equation~\ref{eq:waitint}(b) and Equation~\ref{eq:waitint}(d). Equation~\ref{eq:waitint}(b) ensures that this is the case where the robot has reached the location before the task object. Equation~\ref{eq:waitint}(d) sets the $pos\_time_{i,j}$ to the time at which the task object can be picked up by the robot. After a robot drops the task at $loc\_time_{m,j-1}$ time, any other robot will take at least 1 unit of time to reach that location and 1 more unit to pick up the task from the intermediate location.

Constraints formulation for $r_i$ dropping one of the task objects it carries to one of the $intermediate\ blocks$ in action step $j$ is shown below.
\begin{subequations}
\begin{align}
        drop&\_intermediate(r_i,t_m,i_n,j) \equiv \nonumber \\
        & being\_carried_{m,j-1} = i \\
        \wedge \ & pos_{i,j} = n\ \wedge being\_carried_{m,j} = -1 \\ 
        \wedge \ &  pos\_time_{i,j} = pos\_time_{i,j-1} + \nonumber\\
        & \qquad dist(pos_{i,j-1},i_n) + 1 \\ 
        \wedge \ & loc_{m,j} = n\ \wedge action_{i,j} = m \\ 
        \wedge \ & loc\_time_{m,j} = pos\_time_{i,j} 
\end{align}
\label{eq:dropint}
\end{subequations}
Equation~\eqref{eq:dropint} is similar to Equation~\eqref{eq:drop} as dropping at the intermediate location is similar to dropping at the task's goal location.

Moreover, We need to add constraints to update $loc\_time$ in Equation~\eqref{eq:pick},~\eqref{eq:drop} and~\eqref{eq:consistency}.
Finally, we have to change Equation~\eqref{eq:combined} to 
\begin{align}
    \bigwedge_{r_i \in \mathcal{R}} &\bigwedge_{j=1}^{Z}\ \bigg( stay(r_i,j)  \ \vee  
    \bigvee_{k \in\{{s_i}\}\cup LOC} \Big((pos_{i,j-1} = k) \ \wedge \nonumber \\
    &\ \ \ \big( return(r_i, j) \ \vee \nonumber \\
    &\ \ \ \bigvee_{t_m \in \mathcal{T}} \big(pick(r_i,t_m,j) \ \vee drop (r_i,t_m,j) \ \vee  \nonumber \\
    &\ \ \  \bigvee_{i_n\in\mathcal{I}}(pick\_intermediate(r_i,t_m,i_n,j) \ \vee \nonumber \\ 
    &\ \ \ \ \ \ \ \ \ \   wait\_intermediate(r_i,t_m,i_n,j) \ \vee \nonumber \\
    &\ \ \ \ \ \ \ \ \ \ drop\_intermediate(r_i,t_m,i_n,j)\ )\ \big)\Big)\bigg)  
\label{eq:combinedint}
\end{align}
The final set of constraints is obtained as the conjunction of constraints capturing the initial states and those in Equation \eqref{eq:combinedint},~\eqref{eq:consistency} and~\eqref{eq:endloc}.

\subsubsection{Other operational constraints}
In our task planning framework, we can easily add other operational constraints. The constraints can be mainly of two types based on their association with time. If the constraint is associated with time, e.g., deadline, we need to handle the constraint in Task Planner as well as Path Planner. However, constraints like capacity are not related to time and can be handled through Task Planner only. We have added two constraints to demonstrate both types. 
 
\textit{Capacity constraints.}
We can assign specific weights to task objects and specific weight-carrying capacities to robots. This constraint is independent of time, so it needs to be handled in Task Planner only. Let the variable $capacity_{i,j}$ denote the weight carrying capacity of robot $r_i$ in action-step $j$ and $weight_l$ denote a constant weight of object for task $t_l$. Now, we add the following constraint to all the sets of constraints involving a pickup:
\begin{align}
capacity_{i,j-1} & \ge weight_{l} \ \wedge \nonumber \\ 
& capacity_{i,j} = capacity_{i,j-1} - weight_l
\end{align}
This checks for weight satisfiability before assigning a task to the robot and updates the weight-carrying capacity of the robot after picking it up. Similarly, for all the set of constraints involving a drop operation (Equation~\eqref{eq:drop},~\eqref{eq:dropint}), we add:
\begin{align}
capacity_{i,j} = capacity_{i,j-1} + weight_l
\end{align}
This updates the weight-carrying capacity of the robot after dropping.

\textit{Deadline constraints.} We can also add a specific deadline $deadline_m$ to each task $t_m$ by adding the following constraint for each task in Equation~\eqref{eq:endloc}. This constraint is related to time, so it needs to be handled in Task Planner as well as Path Planner.
\begin{align}
loc\_time_{m,Z} \le deadline_m.
\end{align}

\subsubsection{Exclusion}
We provide a way to add an already found task assignment $\mathcal{A}$ as an exclusion to the SMT planner so that the task planner finds the best solution excluding the already found assignments. Let $POS_{i, j}$ be the position of robot $r_i$ and action step $j$ in the existing solution. 
\begin{align}
\label{eq:exclusion}
    \bigvee_{r_i \in \mathcal{R}} \big( \bigvee_{j \in \mathcal{Z}} ( pos_{i, j} \ne POS_{i, j}) \big)  
\end{align}
To add the existing solution as an exclusion, we have added Equation~\ref{eq:exclusion} to the set of constraints in the SMT solver.

\subsubsection{Objective function}

We present the two cost functions related to the total cost and makespan of the trajectories.
\noindent
\begin{enumerate}
    \item \textsf{Total Cost}: Here, we minimize the total work done by all the robots.
    $$\mathtt{minimize} \ (\sum_{r_i\in \mathcal{R}} pos\_time_{i,Z} )$$
    \item \textsf{Makespan}:  Here, we minimize the time required to complete the mission.
    $$ \mathtt{minimize} \ (\max_{r_i\in \mathcal{R}} \ pos\_time_{i,Z} )$$
\end{enumerate}
}
}

The value of $Z$ must be set such that it satisfies the condition $Z \ge 1 + \lceil |\mathcal {T}| / |\mathcal{R}| \rceil *2$ for the problem to be solvable. To search through all possible task assignments ignoring load balancing among robots, $Z \ge 1 + |\mathcal{T}|*2$.

The task planner uses a binary search algorithm to optimize the cost function guided by the SMT constraints. 
Note that modern SMT solvers like Z3~\cite{Z3_Moura} provide a mechanism to solve a minimization problem directly within the solver. However, our experience shows that attempting to solve an optimization problem directly using an SMT solver often fails to succeed within a reasonable time. In contrast, the binary search-based optimization method can successfully produce the result within a bound.

\subsection{Path Planning} 
\label{sec-motionplanning}

For the path planner, we adopt the CBS-PC algorithm~\cite{zhang2022multi} for multi-agent pathfinding for precedence-constrained goal sequences. CBS-PC uses Multi-Label A*~\cite{grenouilleau2019multi} as its low-level planner. Multi-Label A* can find optimal paths for a sequence of goal locations. As we deal with intermediate drops and pickups, the intermediate pickup must be executed after the intermediate drop for the same task. This is taken care of by the precedence constraints presented in the algorithm. 
We also introduce the following enhancements to the basic CBS-PC algorithm: (i) makespan optimization criteria along with the sum of total costs, (ii) inclusion of deadlines support for goals and checkpoints, and (iii) handling empty goals as the task planner may not assign tasks to some robots.


\longversion{\section{Evaluation}
We evaluate our planning methodology on various instances of warehouse pick-and-drop application scenarios.

\subsection{Experimental Setup}

For all our experiments, we use a desktop computer with an i7-4770 processor with a 3.90\,GHz frequency and 12\,GB of memory. We use \textsf{Z3} SMT solver~\cite{Z3_Moura} from Microsoft Research to solve task-planning problems. For MA*-CBS-PC, we adapt the C++ code provided by \cite{zhang2022multi} with appropriate modifications. 
The source code of our implementation is available at \url{https://github.com/iitkcpslab/Opt-ITPP}.

\begin{figure}[t]
\begin{flushleft}
{
    \label{fig:predef_map}
    \includegraphics[height=4cm, width=4cm]{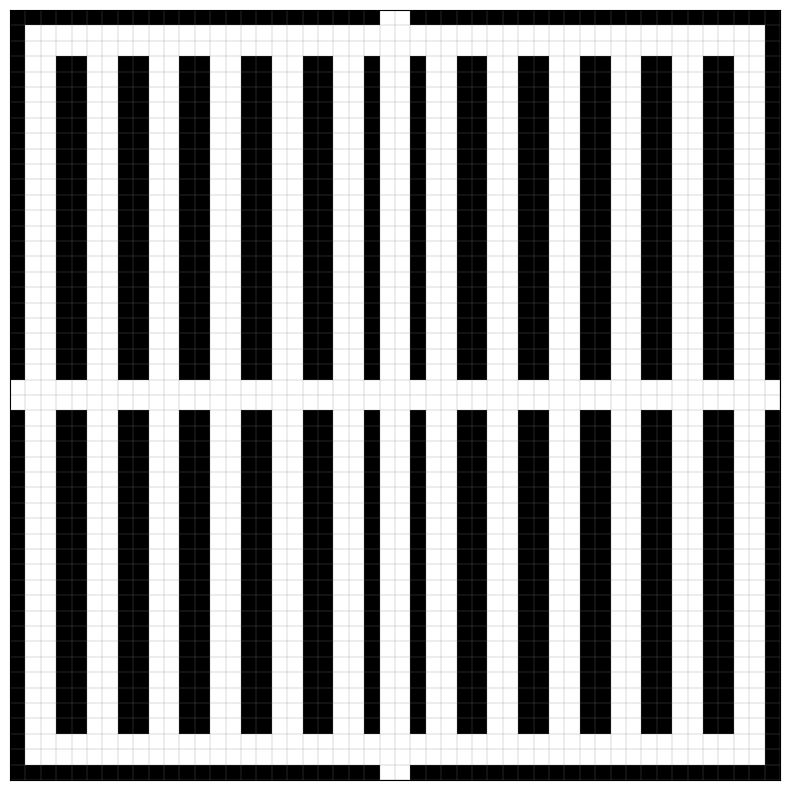}
}
{
    \includegraphics[height=4cm, width=4cm]{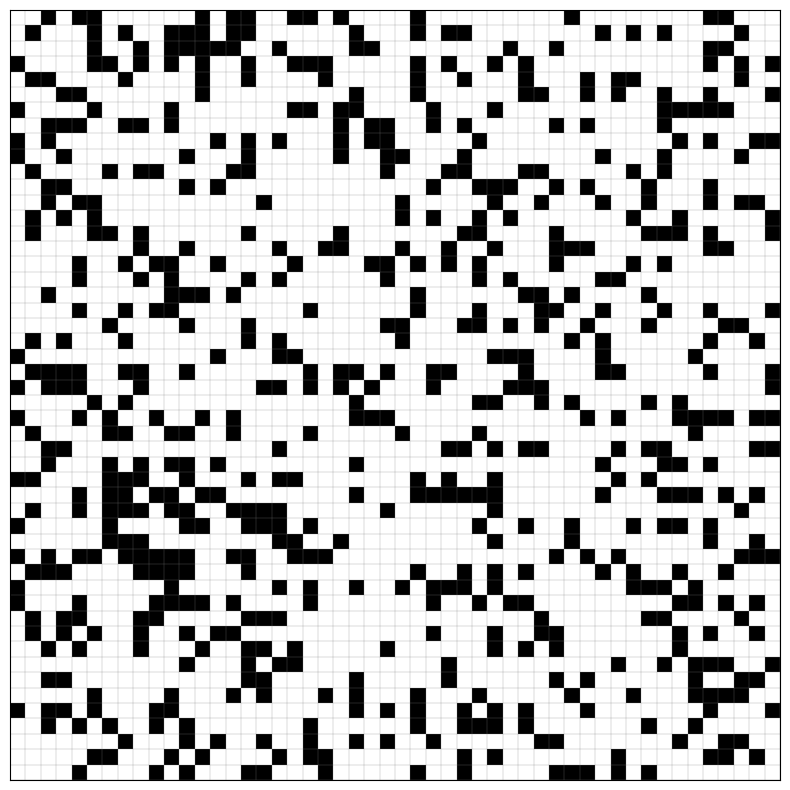}
    \label{fig:random_map}
}
\caption{Predefined (left) and Randomly generated (right) 50x50 map}
\label{fig:extra_maps}
\end{flushleft}
\end{figure} 

For any data point, we take the average of the results for multiple generated scenarios where the initial location of the robots and the task locations are generated randomly. For each experiment, we have used 20 different examples using predefined as well as randomly generated maps as shown in Figure~\ref{fig:extra_maps}. The first one resembles a warehouse, and the second is one for which the obstacles are generated randomly.

In our experiments, we consider two planners: one optimizes the makespan (opt\_makespan), and the other optimizes the total cost (opt\_cost). 
In all the tests, we have set the timeout as $3600\si{\second}$. In the plots, for all the cases where the planner fails to solve the problem in $3600\si{\second}$, we take its computation time as $3600\si{\second}$ and the metric value as the average of the values for the instances the planner can solve successfully. 

\subsection{Evaluation of Task Planner}

\label{appendix_tp_eval}
In this section, we evaluate our SMT-based task planner. To evaluate the task planner, we use a typical warehouse-like workspace and randomly generated location pairs. We evaluate our Task Planner extensively for various settings by varying the number of robots and tasks, map size, and the number of actions (Z).  We use the average of the metrics obtained by executing the planner on ten problem instances each. 

\begin{figure}[t]
\begin{flushleft}
{
    \label{fig:tp_varyrt_om_ct}
    \includegraphics[height=3.1cm, width=4.1cm]{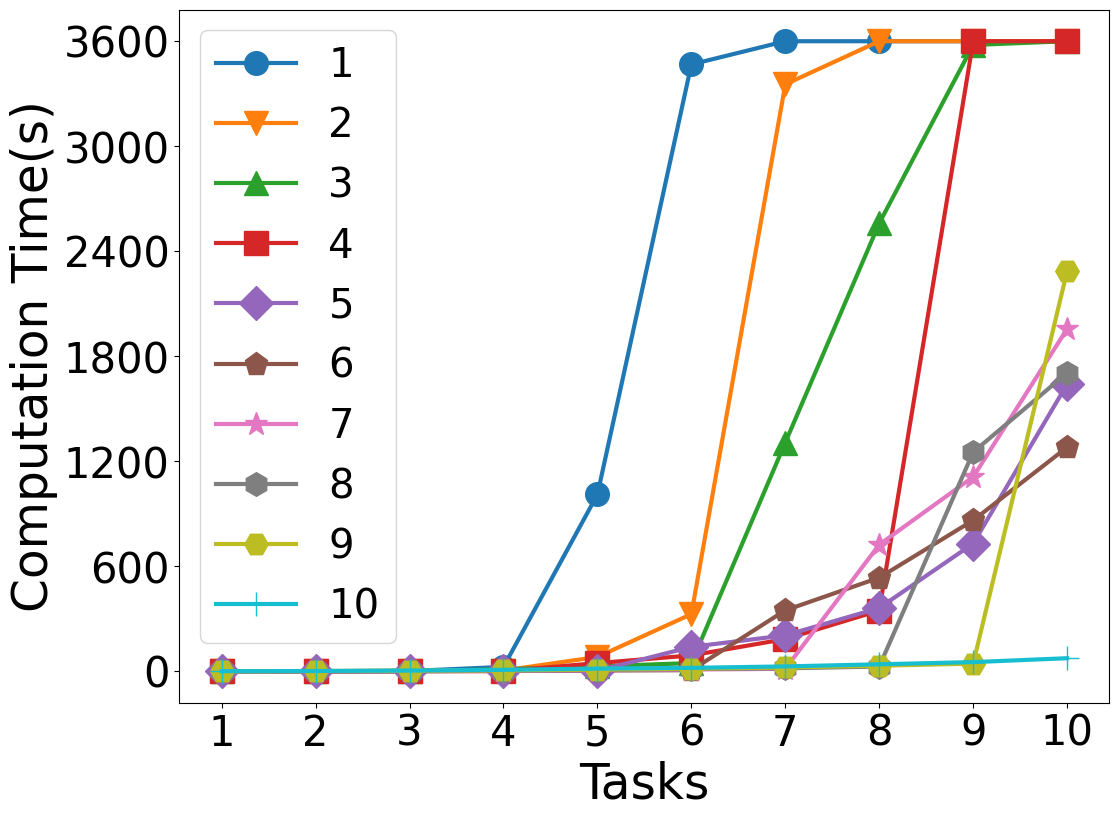}
}
{
    \label{fig:tp_varyrt_om_ms}
    \includegraphics[height=3.1cm, width=4.1cm]{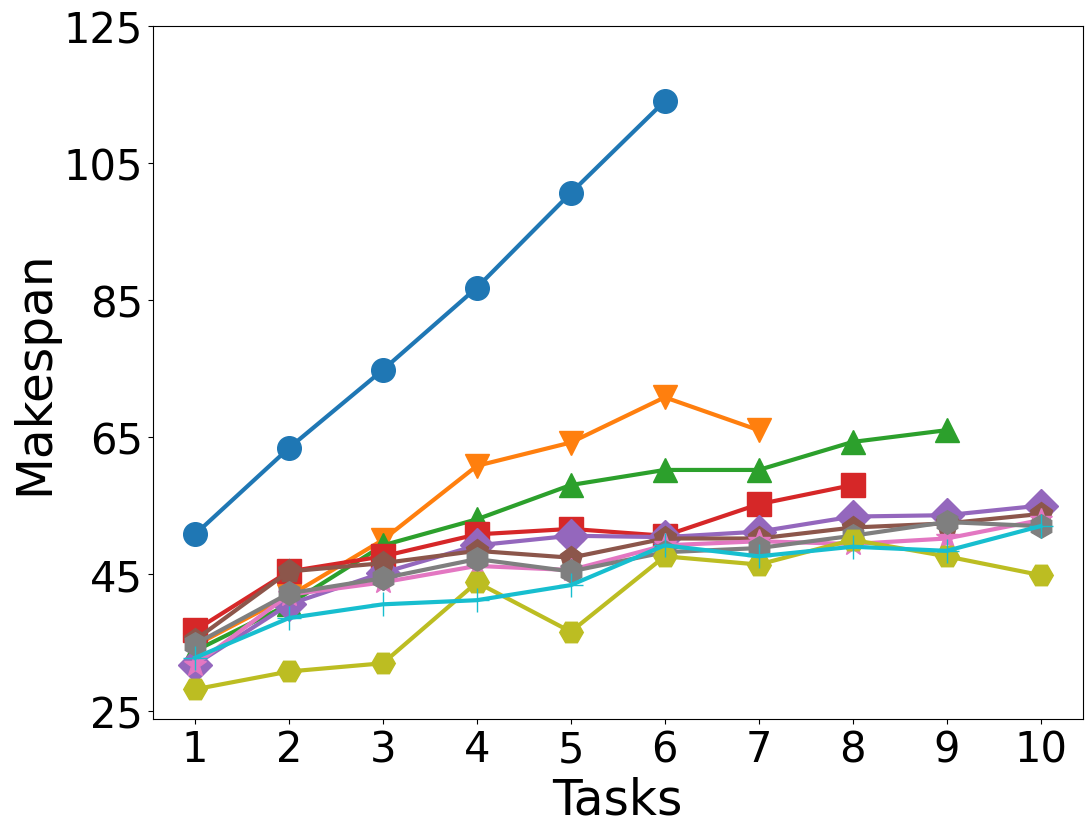}
} 
\caption{Task Planner : The effect of increasing the number of robots (shown in legends) and the number of tasks for task planner with makespan optimization criteria on a) Computation Time (left) and b) Makespan (right)}
\label{fig:tp_varyrt_om}
\end{flushleft}
\end{figure} 


\begin{figure}[t]
\begin{flushleft}
{
    \label{fig:tp_varyrt_oc_ct}
    \includegraphics[height=3.1cm, width=4.1cm]{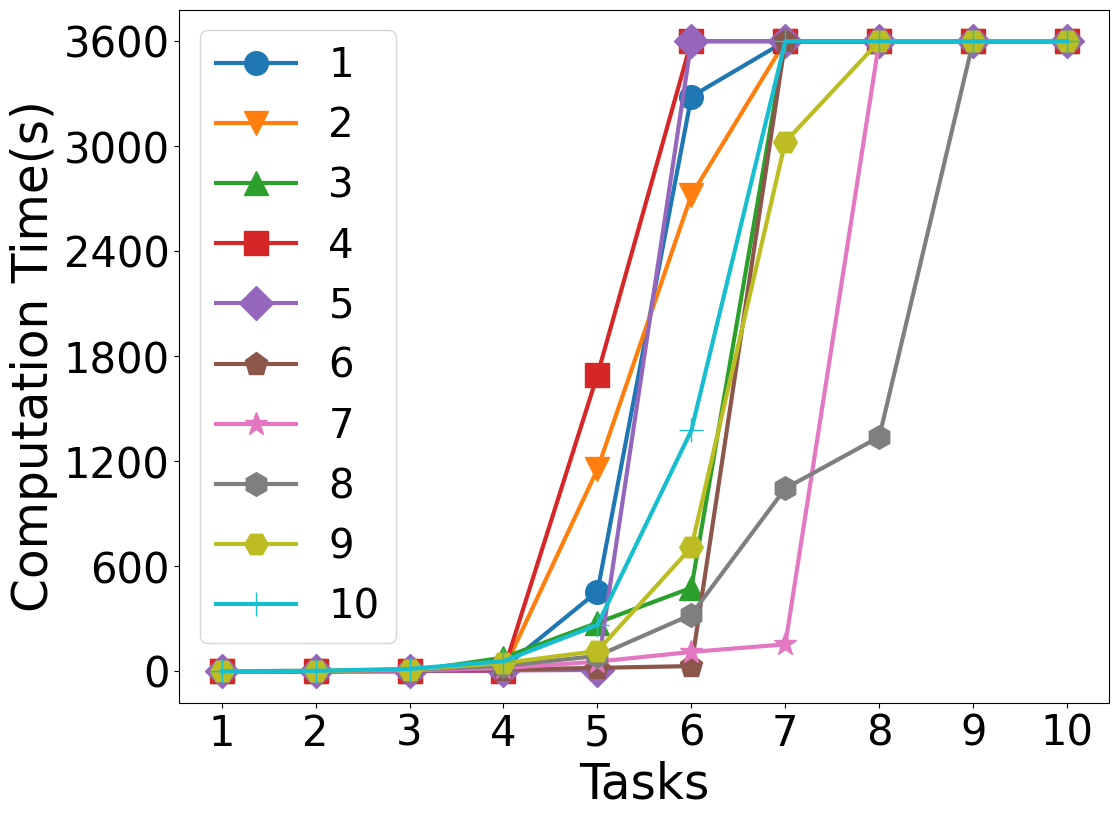}
}
{
    \label{fig:tp_varyrt_oc_tc}
    \includegraphics[height=3.1cm, width=4.1cm]{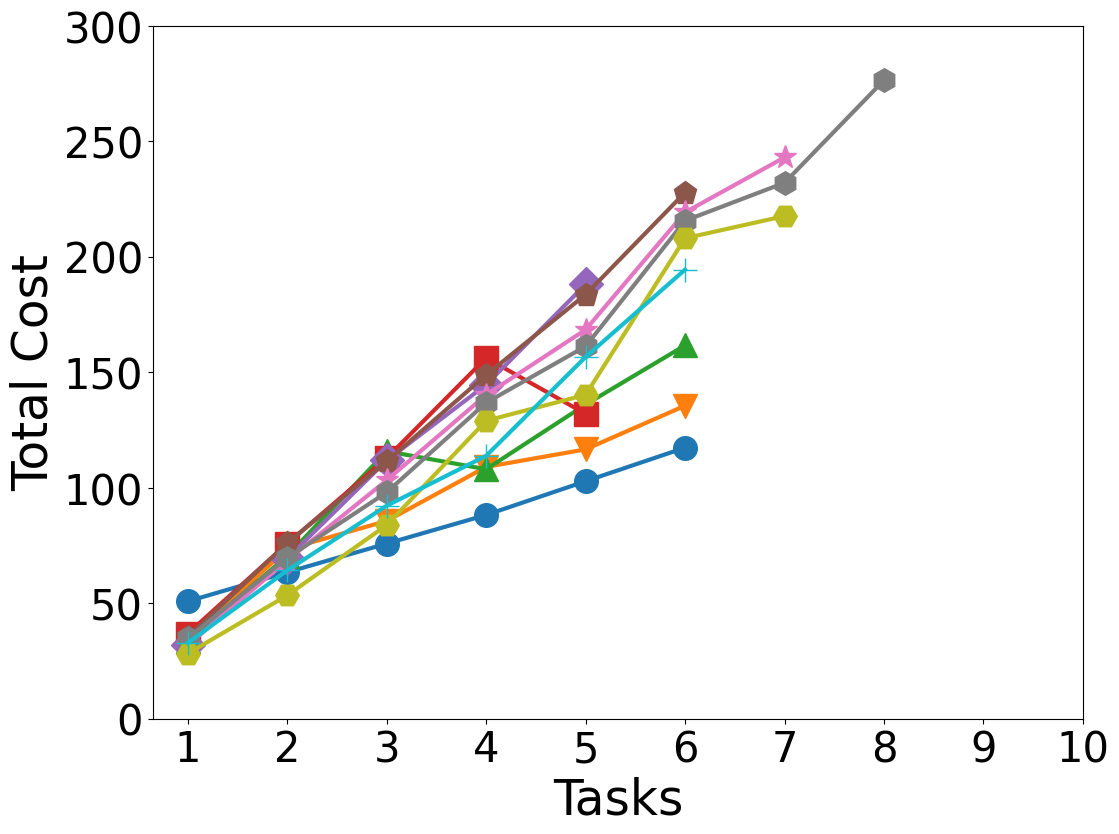}
} 
\caption{Task Planner : The effect of increasing the number of robots (legends) and the number of tasks for task planner with total cost optimization criteria  on a) Computation Time (left) and b) Total Cost (right)}
\label{fig:tp_varyrt_oc}
\end{flushleft}
\end{figure}

\begin{figure}[t]
\begin{flushleft}
{
    \label{fig:tp_varyw_ct}
    \includegraphics[height=3.1cm, width=4.1cm]{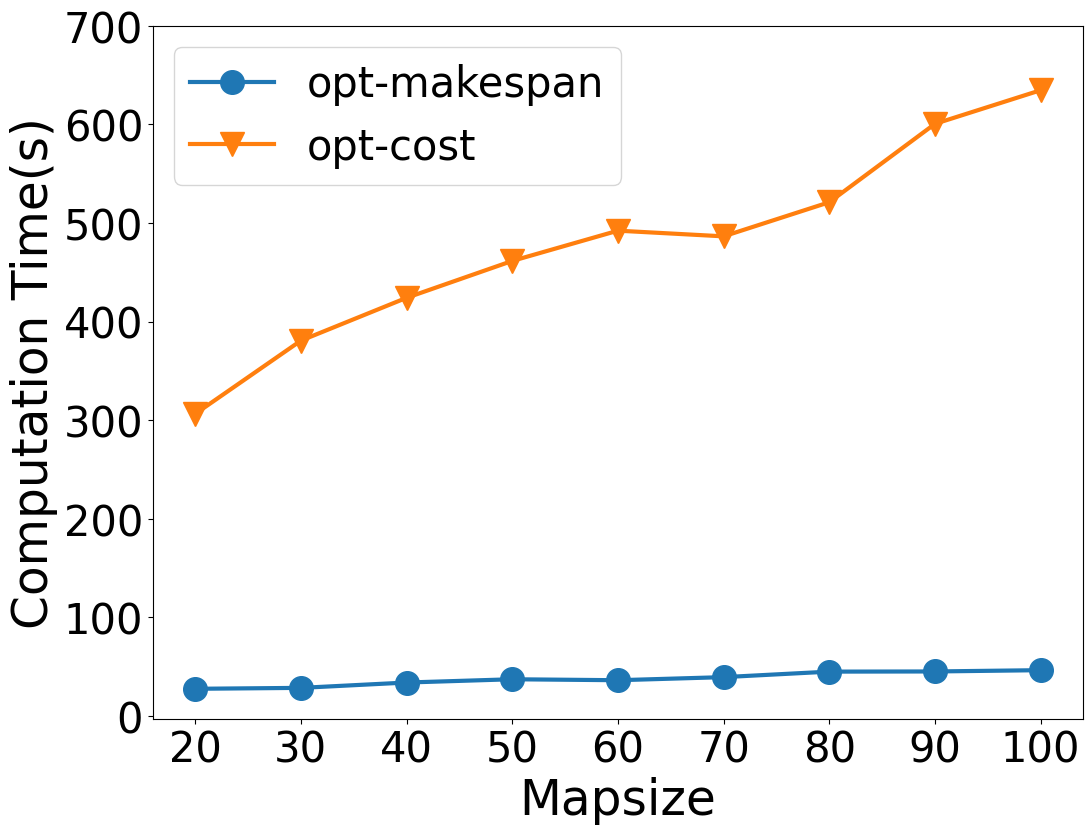}
}
{
    \label{fig:tp_varyw_ms}
    \includegraphics[height=3.1cm, width=4.1cm]{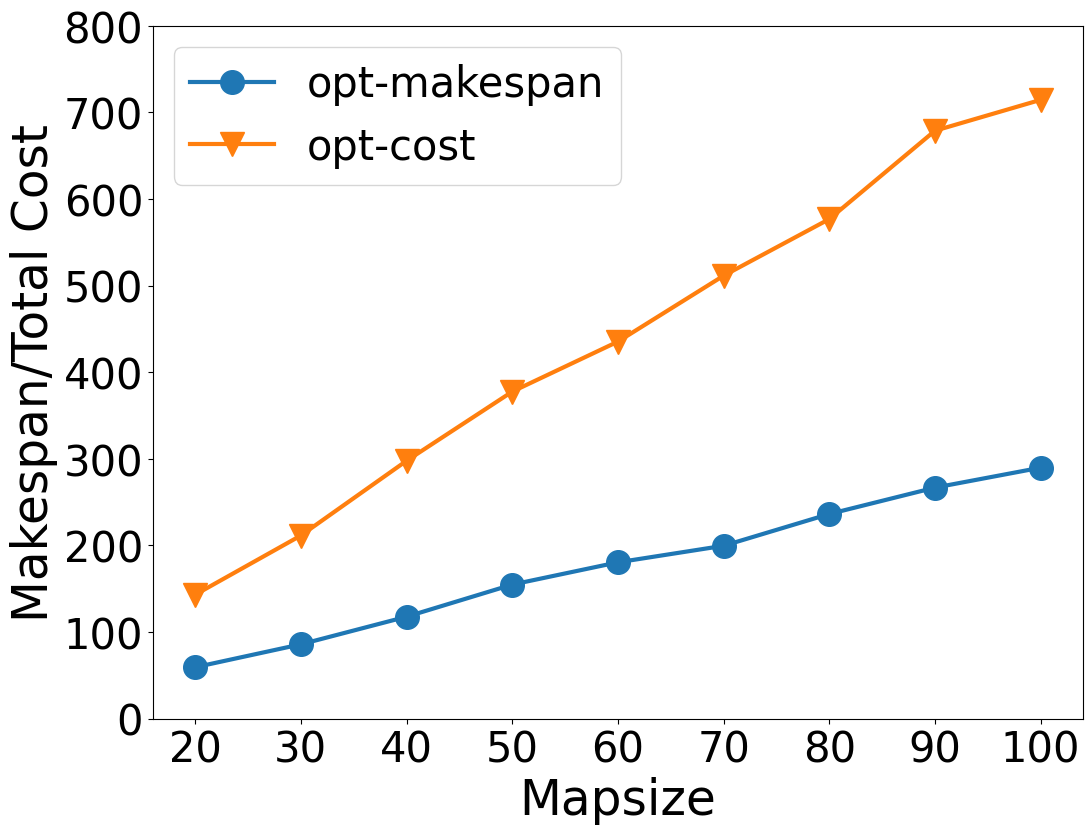}
} 
\caption{Task Planner : The effect of changing workspace size for various optimization modes (shown in legends)  on a) Computation Time (left) and b) Makespan/Total Cost (right)}
\label{fig:tp_varyw}
\end{flushleft}
\end{figure}

\begin{figure}[t]
\begin{flushleft}
{
    \label{fig:tp_varyz_ct}
    \includegraphics[height=3.1cm, width=4.1cm]{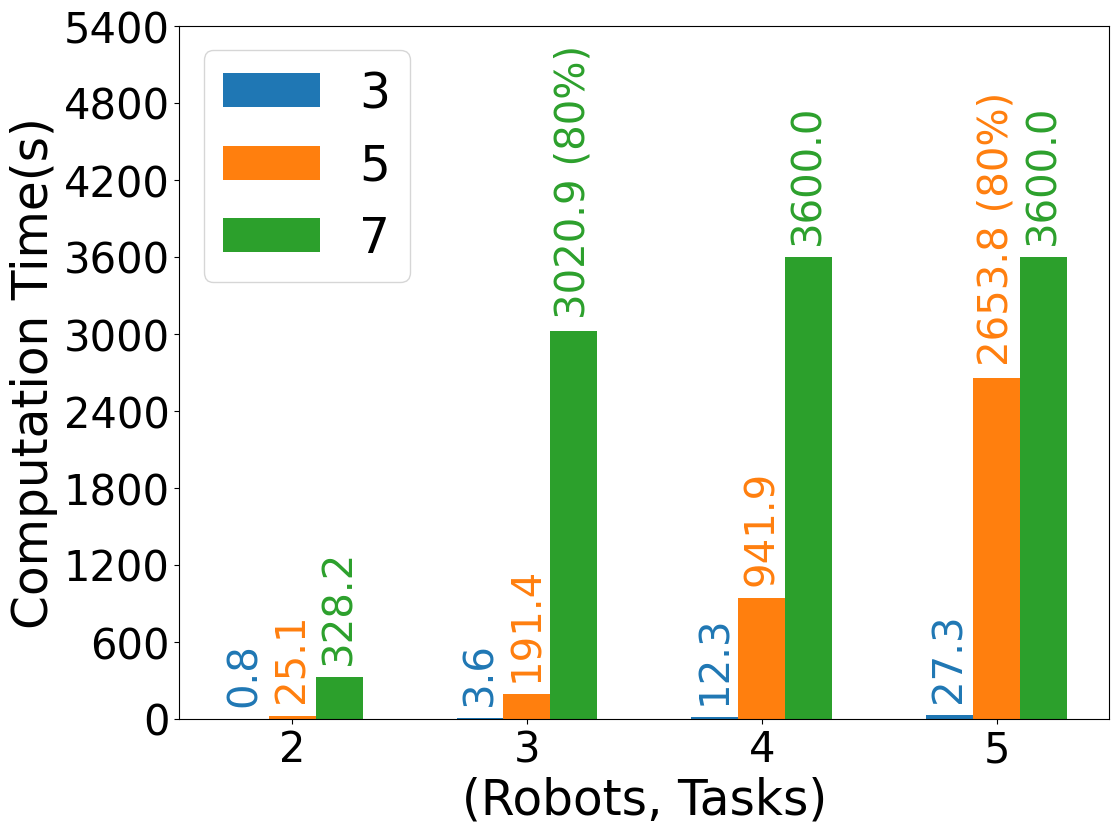}
}
{
    \label{fig:tp_varyz_ms}
    \includegraphics[height=3.1cm, width=4.1cm]{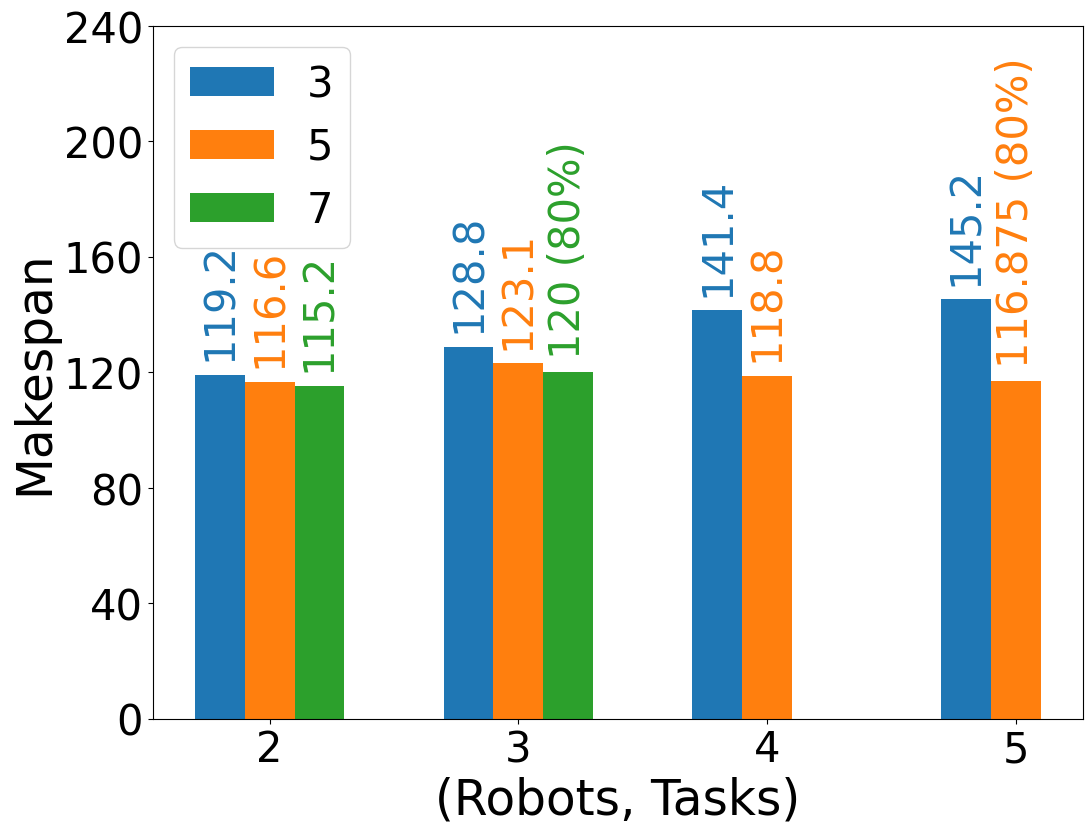}
} 
\caption{Task Planner : The effect of increasing $Z$ (shown in legends) on a) Computation Time (left) and b) Makespan (right)}
\label{fig:tp_varyz}
\end{flushleft}
\end{figure} 

\subsubsection{Task Planning without collaboration} 
We evaluated our task planner for varying number of robots and tasks with our optimization criteria. We employ a $20 \times 20$ workspace for these evaluations with the minimum satisfiable $Z$ for each number of robots and the tasks pair. Figure~\ref{fig:tp_varyrt_om}a shows how the computation time varies with the increase in the number of robots and the number of tasks for makespan optimization criteria.  The plot shows that the computation time is very low when the number of tasks is less than or equal to the number of robots. For each robot, we observe an increase in computation time for an increase in tasks. But, for each increase in the number of action steps denoted by $Z$, we observe a substantial increase in computation time. This increase in $Z$ reflects a corresponding rise in the tasks assigned per robot, approximated as the rounded value of the number of tasks divided by the number of robots.

Figure~\ref{fig:tp_varyrt_oc}a is a similar plot for total cost optimization criteria. From the plot, we can observe that except for a single robot, the computation time for optimizing total cost is higher than optimizing makespan.  The difference keeps increasing with higher robot and task counts. Also, the planner cannot hand more than 8 tasks with any number of robots for total cost optimization. The results indicate that our task planner with total cost optimization is not as scalable as optimizing makespan for varying numbers of robots and tasks.

Figure~\ref{fig:tp_varyrt_om}b shows the change in makespan, and Figure~\ref{fig:tp_varyrt_oc}b shows the change in total cost with the increase in the number of robots and tasks. The makespan improves with the increased number of robots as the tasks get distributed between more robots. Generally, the total costs for individual robots should increase linearly with increased number of tasks. The makespan may remain the same for the same $Z$ for individual robots and increase when $Z$ increases. However, we observe many fluctuations and anomalies in our plots as all the problem instances are randomly generated.

Figure~\ref{fig:tp_varyw}a shows the changes in computation time by varying workspace sizes for $3$ robots and $5$ tasks, with both the optimization criteria. From the plot, we can see that the computation time increases slightly with an increase in workspace size for optimizing makespan. For optimizing total cost, there is no monotonous increase in computation time with varying sizes. Generally, the computation time is expected to increase as the range of values to search in the binary search increases with an increase in workspace size. 
Another important observation is that the task planner takes significantly more time to optimize the total cost than the makespan. 

Figure~\ref{fig:tp_varyw}b shows the change in makespan and total cost with the change in workspace size. The plot contains the makespan metric for opt-makespan mode and the total cost metric for opt-cost mode. As expected, both metrics are increasing linearly with an increase in the workspace size.

\subsubsection{Task Planning with collaboration}

 Here, we have experimented on a $50 \times 50$ workspace with multiple values of $Z$, a minimum satisfiable $Z$ ($Z_{min}$) to show no collaboration, and $Z_{min}+2$ and $Z_{min} + 4$ to show collaboration. Each increase of 2 in $Z$ allows each robot to perform two extra actions. So, it can perform additional $drop\_at\_intermediate$ and $pickup\_from\_intermediate$. We have executed our task planner for $n$ robots $n$ tasks for various $Zs$. Figure~\ref{fig:tp_varyz}a and  Figure~\ref{fig:tp_varyz}b show how computation time and makespan vary with different values of $Z$, respectively. With the increase in Z, we see the computation time increase tremendously, but with higher Z, we get plans with a better makespan. For $Z=5$, we hit timeout for some instances with five robots and five tasks. For $Z=7$, we hit timeout for some instances with three robots with three tasks and timeout for all instances with four robots with four tasks and above.

\subsection{Evaluation of Integrated Task and Path Planner}
In this section, we evaluate our integrated task and path planners by comparing it  with a state-of-the-art classical planner ENHSP-20~\cite{scala2020subgoaling}. Since our planner deals with numeric values for capacities and deadlines, we required a classical planner supporting numeric values and providing optimal solutions. 
We explored the possibility of modeling our problem as a constrained TSP problem and utilizing the meta-heuristic algorithm LKH3~\cite{helsgaun2017extension} to get a near-optimal solution. 
However, we did not find any extension of LKH3 that can deal with all the constraints we consider in our problem.
On the other hand, it was quite straightforward to model our exact problem in SMT as well as in ENHSP-20.

In our result plots, in all the instances where time is $3600\si{\second}$, the planner experiences a timeout. We include success percentages as annotations wherever the planner could not solve all the problems. In the plots, for all the cases where the planner faces a timeout,
we take its computation time as $3600\si{\second}$ and the metric value as the average of the values for the instances the planner can solve successfully. 
\begin{figure}[t]
\begin{flushleft}
{
    \includegraphics[height=3.15cm, width=4.1cm]{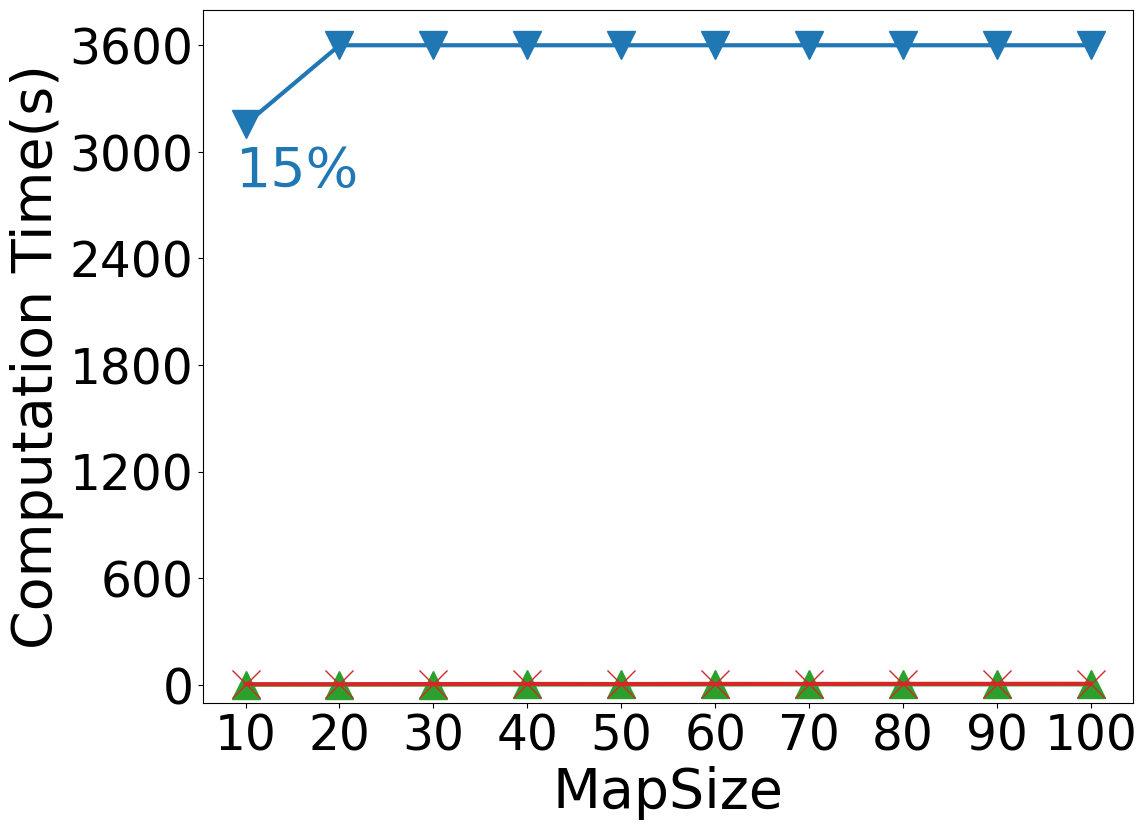}
}
{
    \includegraphics[height=3.15cm, width=4.1cm]{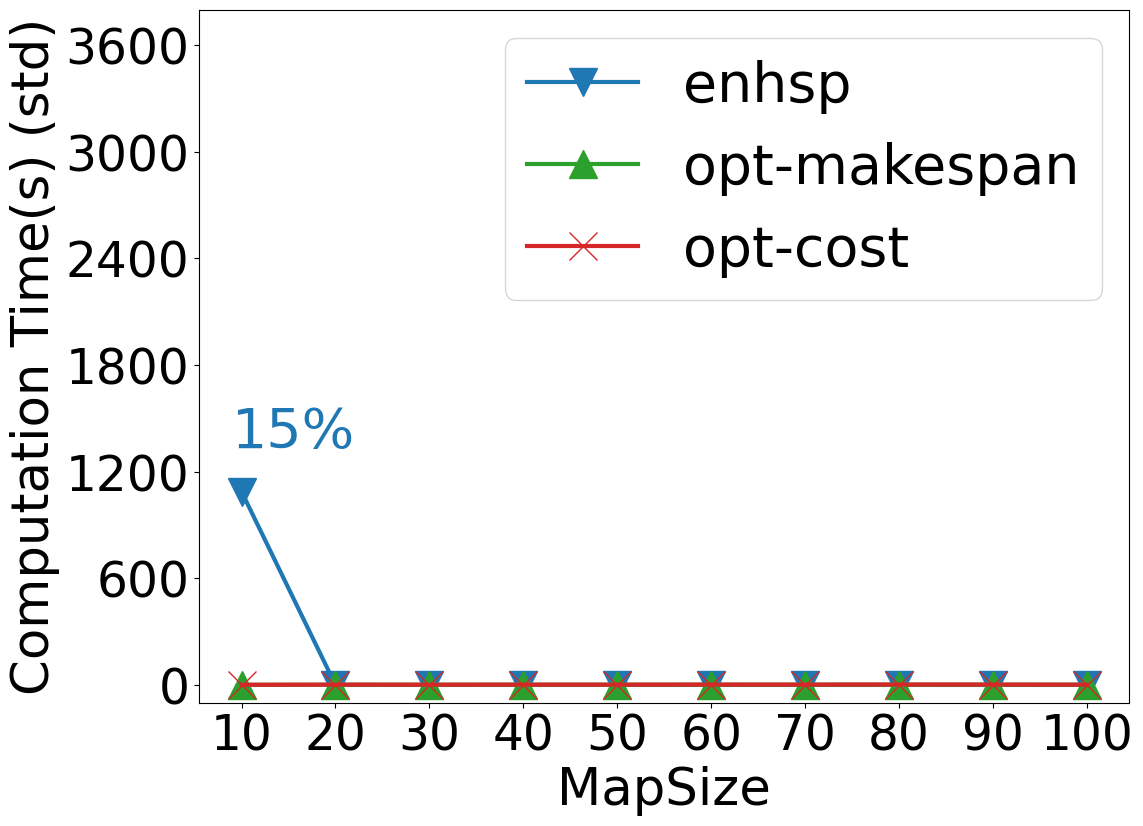}
}
\caption{Comparison of various planners (shown in legends) for varying workspace size on Computation Time. Mean (left) and Standard deviation (right).}
\label{fig:itmp_varyw_ct} 
\end{flushleft}
\end{figure}

\begin{figure}[t]
\begin{flushleft}
{
    \includegraphics[height=3.15cm, width=4.1cm]{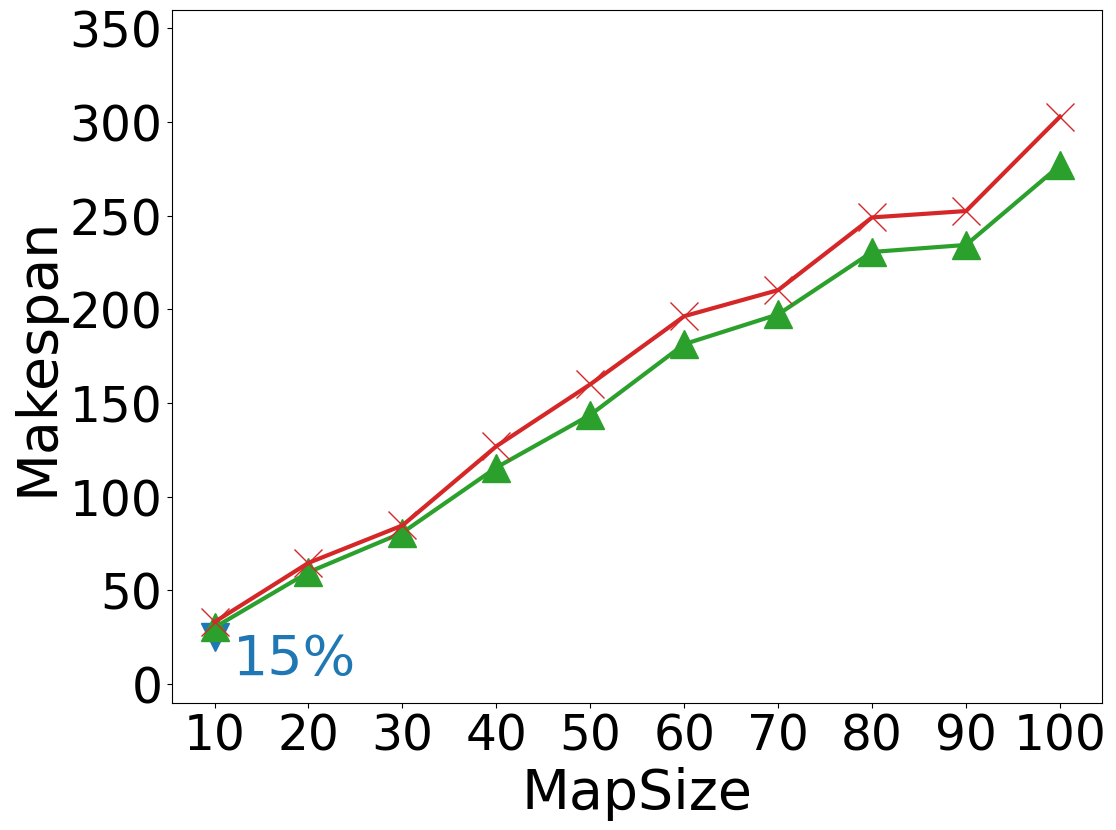}
}
{
    \includegraphics[height=3.15cm, width=4.1cm]{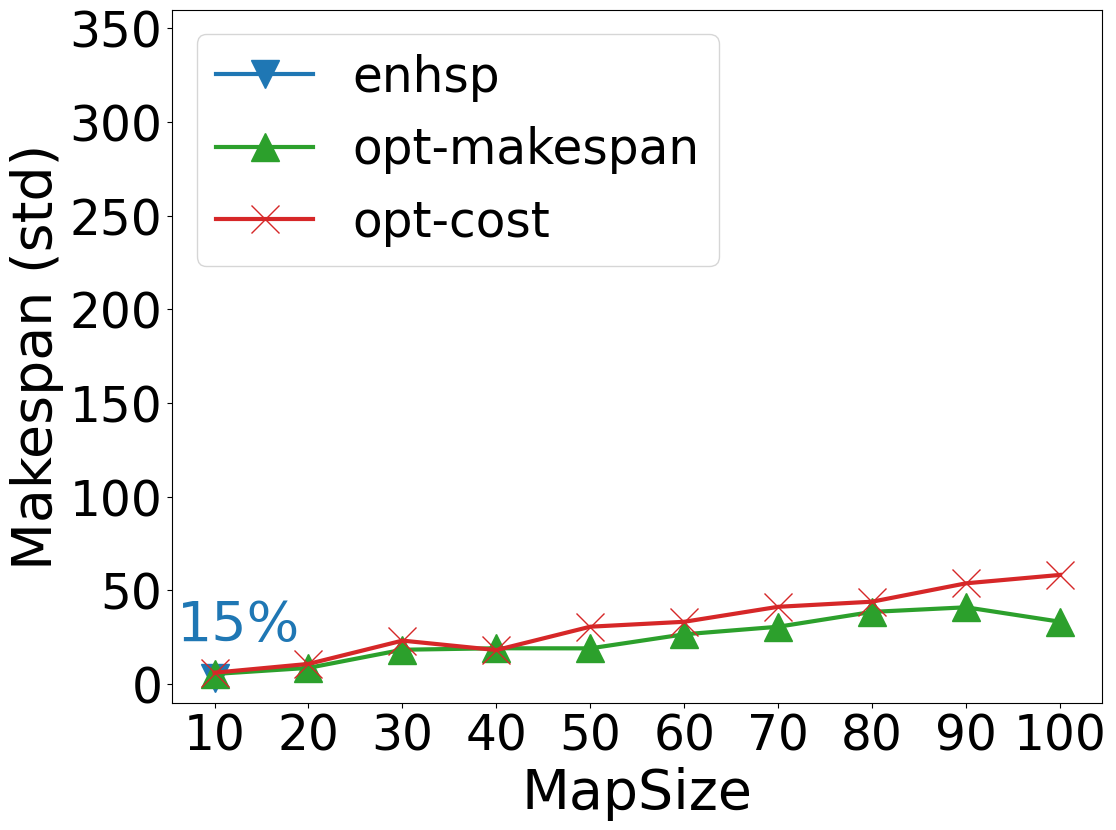}
}
\caption{Comparison of various planners (shown in legends) for varying workspace size on Makespan. Mean (left) and Standard deviation (right).}
\label{fig:itmp_varyw_ms} 
\end{flushleft}
\end{figure}

\begin{figure}[t]
\begin{flushleft}
{
    \includegraphics[height=3.15cm, width=4.1cm]{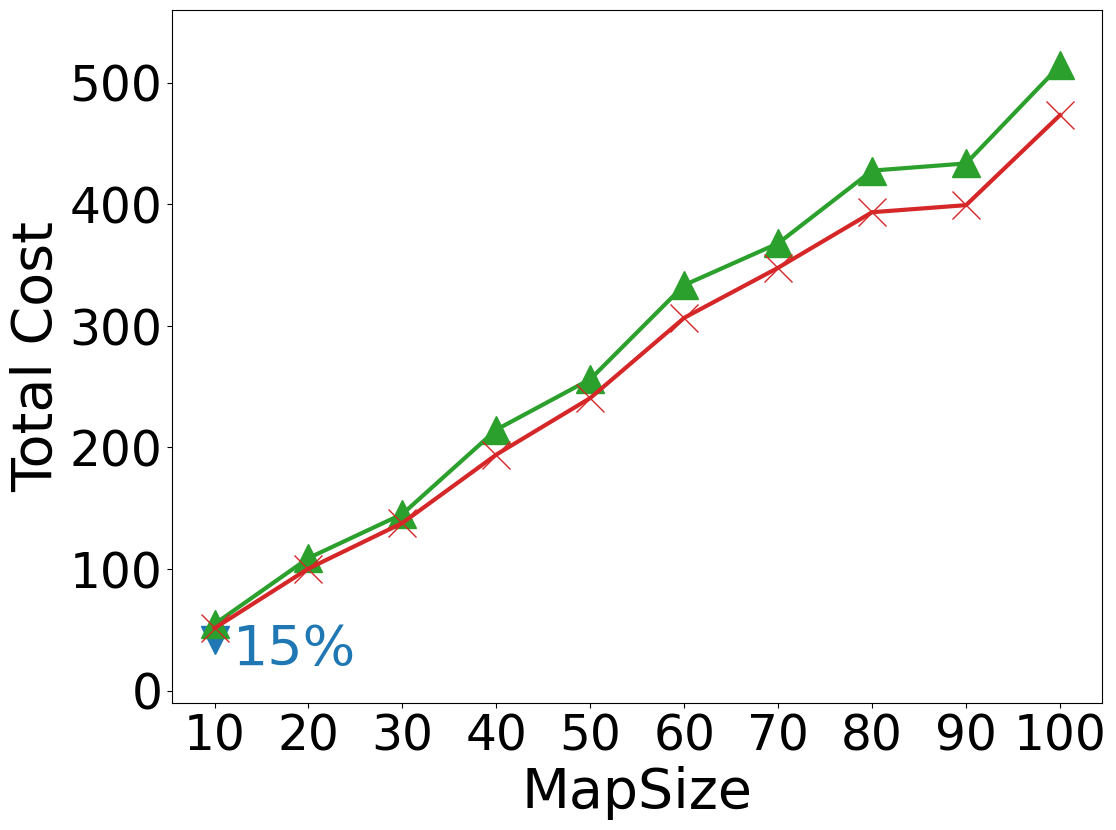}
}
{
    \includegraphics[height=3.15cm, width=4.1cm]{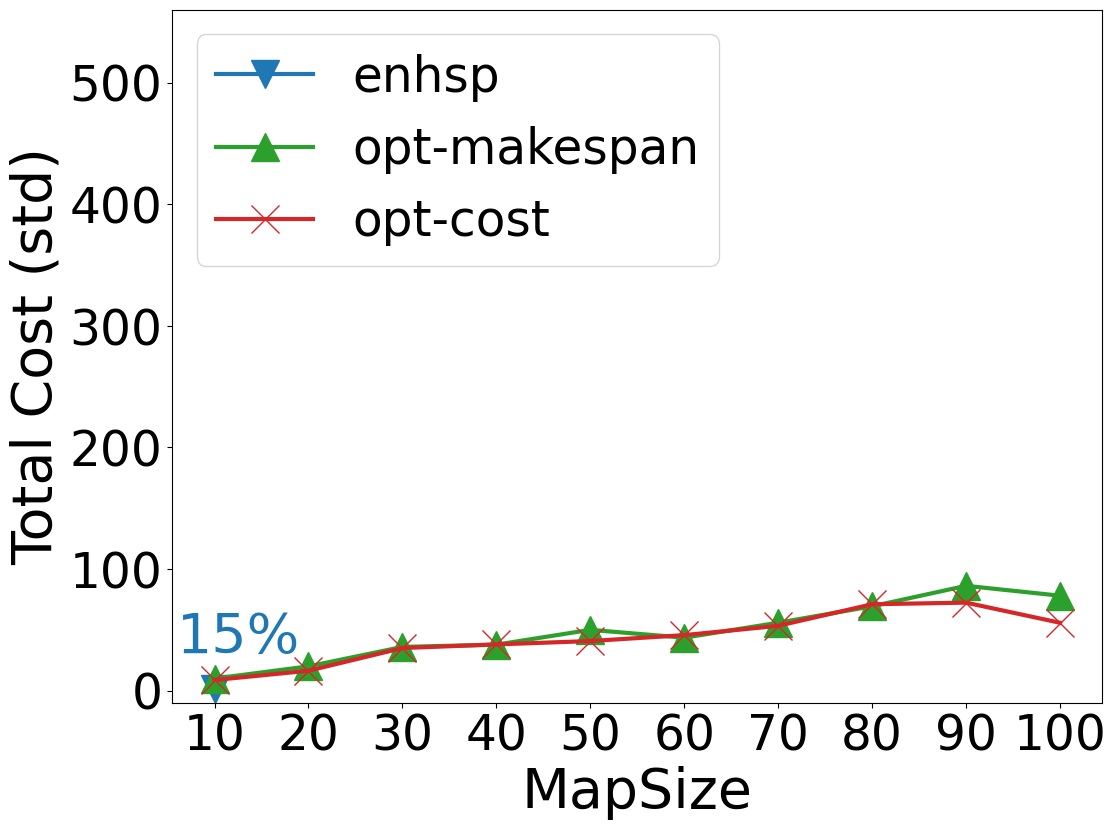}
}
\caption{Comparison of various planners (shown in legends) for varying workspace size on Total Cost. Mean (left) and Standard deviation (right).}
\label{fig:itmp_varyw_tc} 
\end{flushleft}
\end{figure}

\subsubsection{Comparison for varying workspace size}
In this evaluation, we experiment with $2$ robots and $2$ tasks with $Z=5$ for varying workspace sizes ranging from $10 \times 10$ to $100 \times 100$. Figure \ref{fig:itmp_varyw_ct} shows the computation time for varying map sizes for our planners and the ENHSP-20 planner. Our planners are able to solve all the problems in less than a few seconds. The ENHSP planner was able to solve $15\%$ of the problems for the smallest $10 \times 10$ map and was unable to solve any problem with a larger map size. 
Figure~\ref{fig:itmp_varyw_ms} and Figure ~\ref{fig:itmp_varyw_tc} presenting the makespan and the total cost, respectively, is as per the expectations, showing a linear increase in makespan and total cost respectively with an increase in map size.

\begin{figure}[t]
\begin{flushleft}
{
    \includegraphics[height=3.15cm, width=4.1cm]{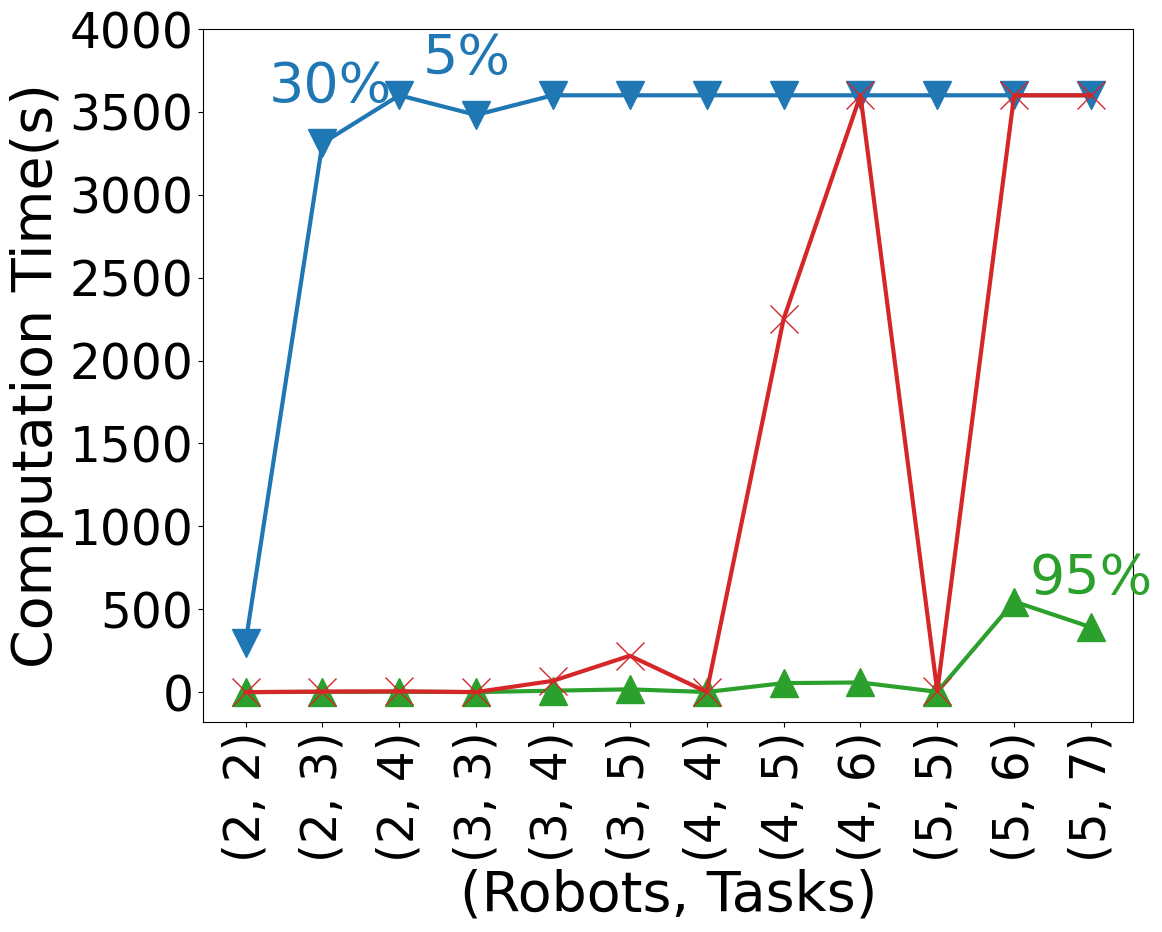}
}
{
    \includegraphics[height=3.15cm, width=4.1cm]{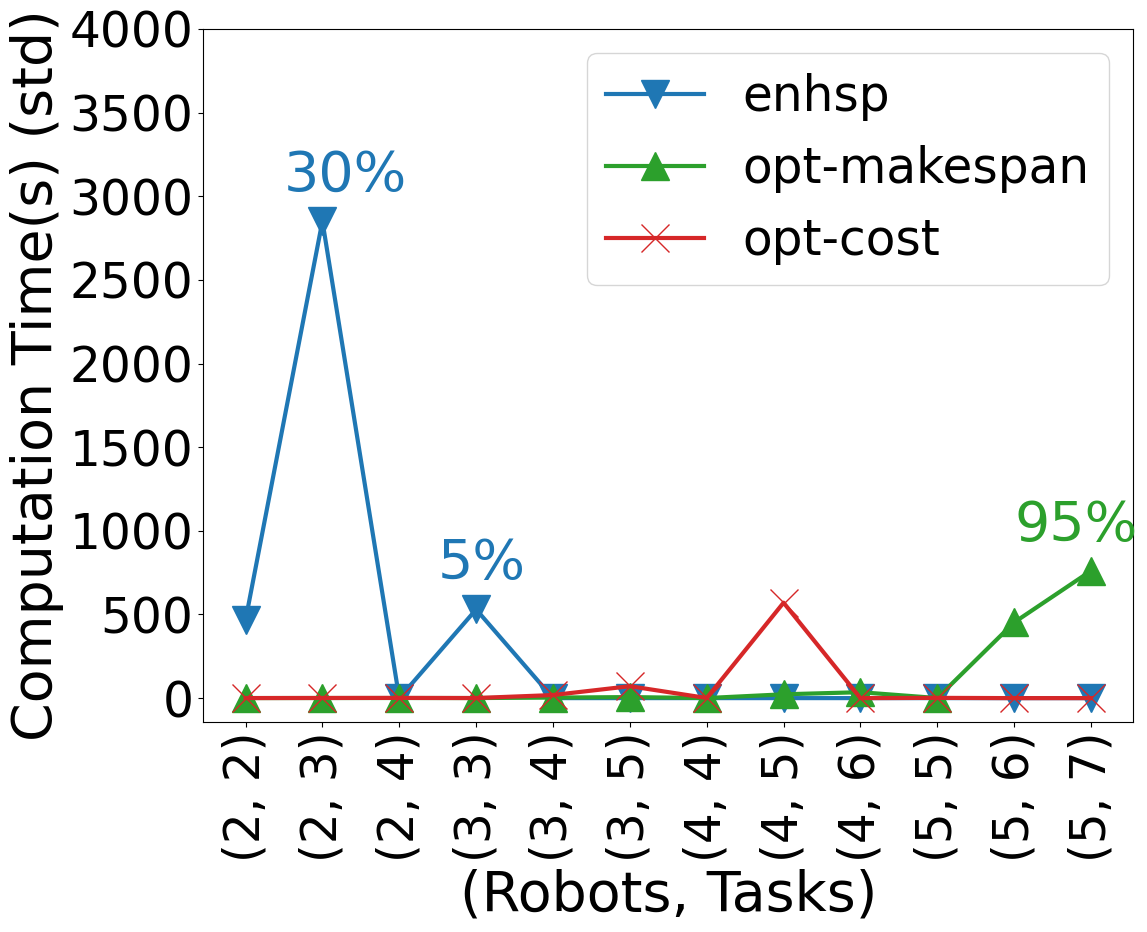}
} 
\caption{Comparison of various planners (shown in legends) for varying Robots and Tasks without Collaboration on Computation Time. Mean (left) and Standard deviation (right).}
\label{fig:itmp_varyrt_ct} 
\end{flushleft}
\end{figure}

\begin{figure}[t]
\begin{flushleft}
{
    \includegraphics[height=3.15cm, width=4.1cm]{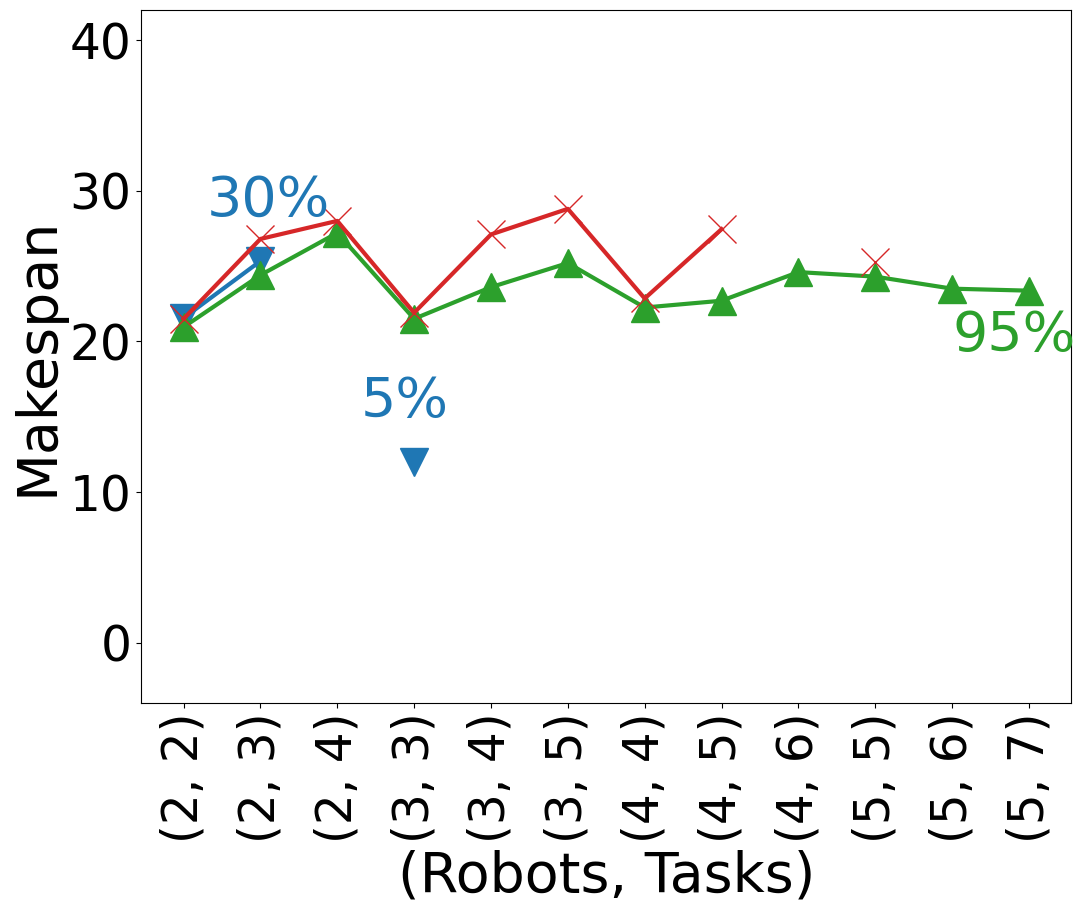}
}
{
    \includegraphics[height=3.15cm, width=4.1cm]{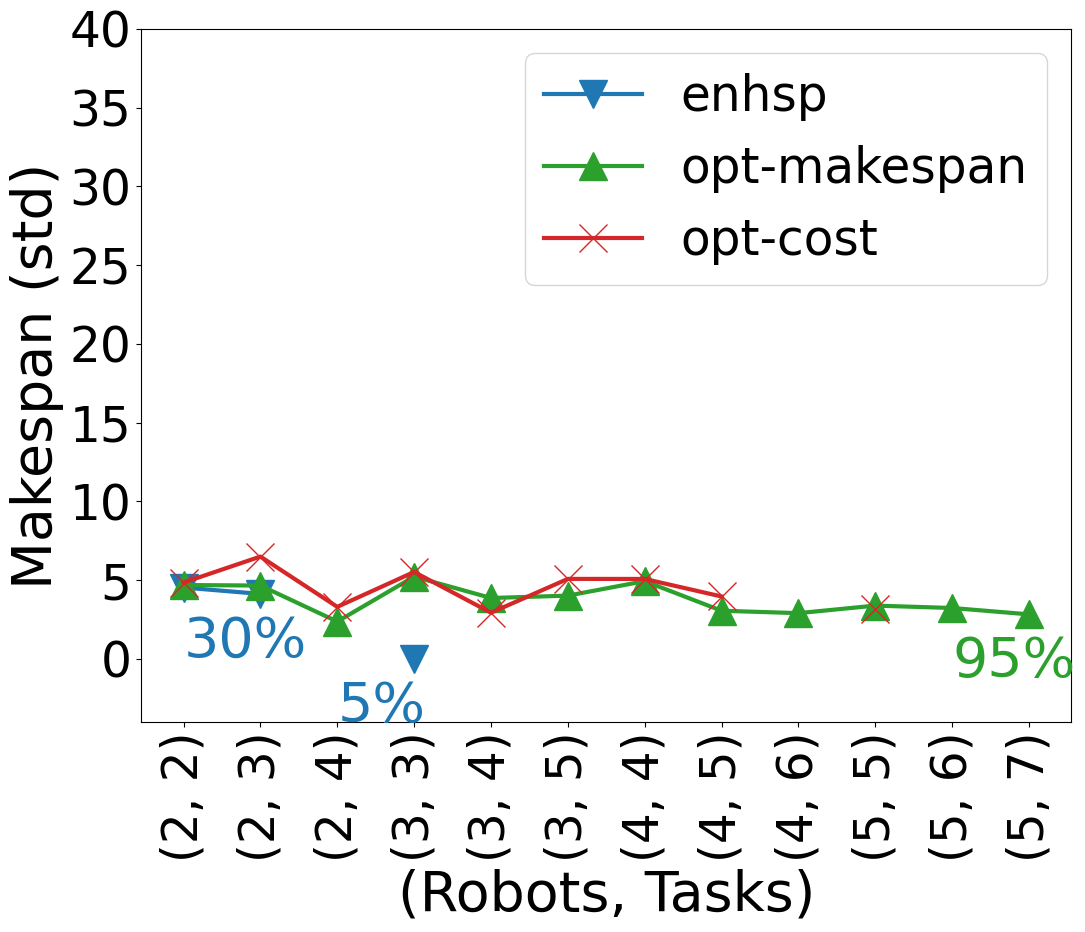}
} 
\caption{Comparison of various planners (shown in legends) for varying Robots and Tasks without Collaboration on Makespan. Mean (left) and Standard deviation (right).}
\label{fig:itmp_varyrt_ms} 
\end{flushleft}
\end{figure}

\begin{figure}[t]
\begin{flushleft}
{
    \includegraphics[height=3.15cm, width=4.1cm]{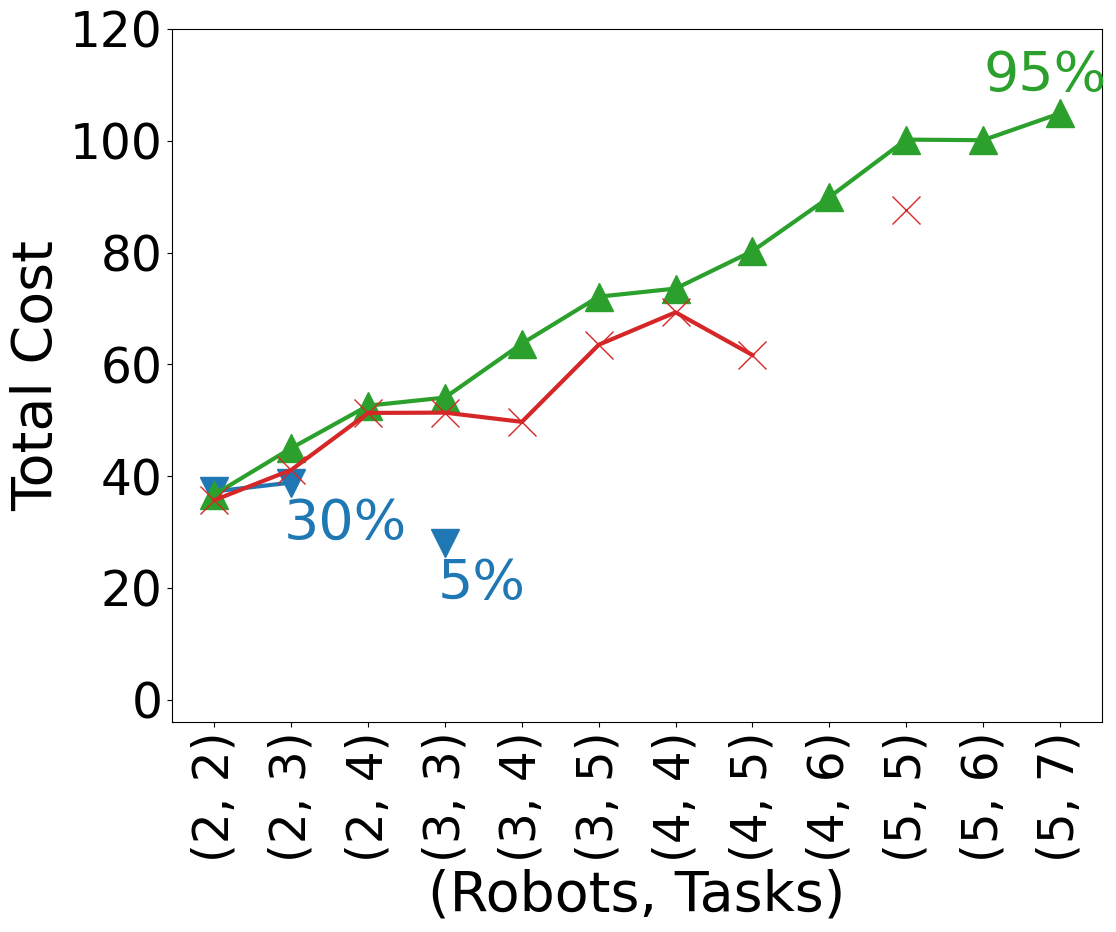}
}
{
    \includegraphics[height=3.15cm, width=4.1cm]{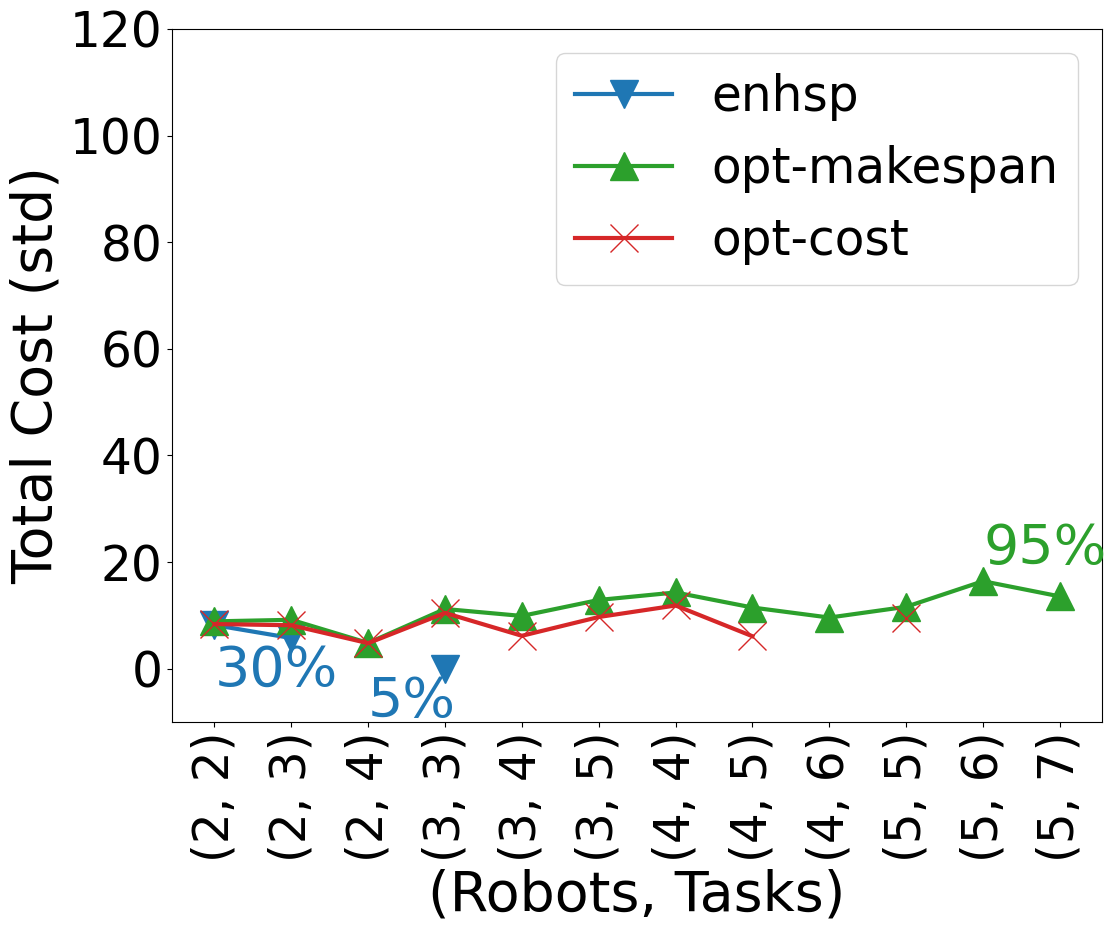}
} 
\caption{Comparison of various planners (shown in legends) for varying Robots and Tasks without Collaboration on Makespan. Mean (left) and Standard deviation (right).}
\label{fig:itmp_varyrt_tc} 
\end{flushleft}
\end{figure}

\subsubsection{Comparison for varying Robots and Tasks without Collaboration}
From the previous evaluation, we observe that the classical planner cannot solve problems for map size more than $10 \times 10$. So, in this experiment, we use maps of size  $9 \times 9$. 
We experiment with 2 to 5 robots and the number of tasks ranging from 2 to 7. Since we aim for a load-balanced solution, we use a minimum satisfiable $Z$ as it forces every robot to perform some work. 
Figure~\ref{fig:itmp_varyrt_ct} shows the computation time for varying number of robots and tasks. The classical planner cannot solve any problem for more than 3 robots. Even for 3 robots, it can solve some of the problem instances. On the other hand, our planners perform significantly better compared to the classical planner. As optimizing total cost is harder for our planner, it starts facing timeout for 6 tasks. Our planner with makespan optimization solves almost all of the problems. It faces timeout for 5\% of the cases for 5 robots and tasks.
Figure~\ref{fig:itmp_varyrt_ms} and Figure ~\ref{fig:itmp_varyrt_tc} denotes a change in makespan and total cost with varying number of robots and tasks. From the plots, we observe that the opt-makespan planner produces better plans than others. Optimizing makespan is more scalable compared to other planners.

\begin{figure}[t]
\begin{flushleft}
{
    \includegraphics[height=3.15cm, width=4.1cm]{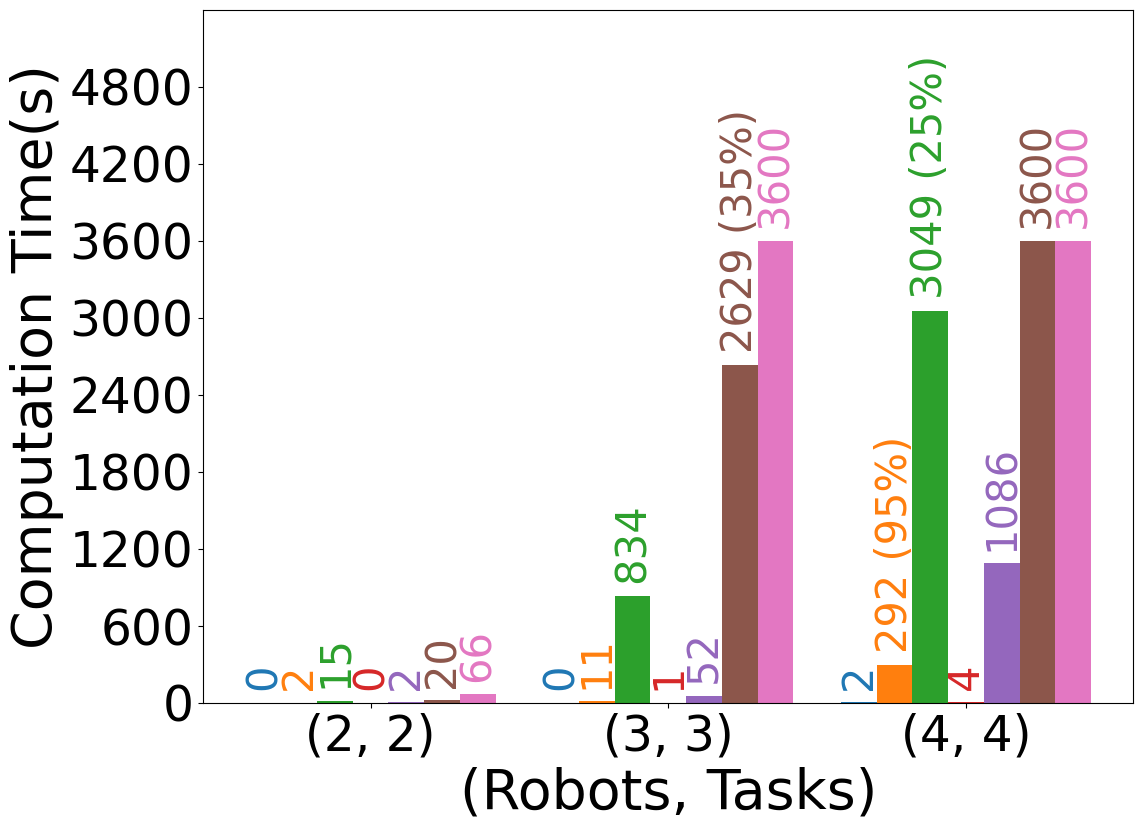}
}
{
    \includegraphics[height=3.15cm, width=4.1cm]{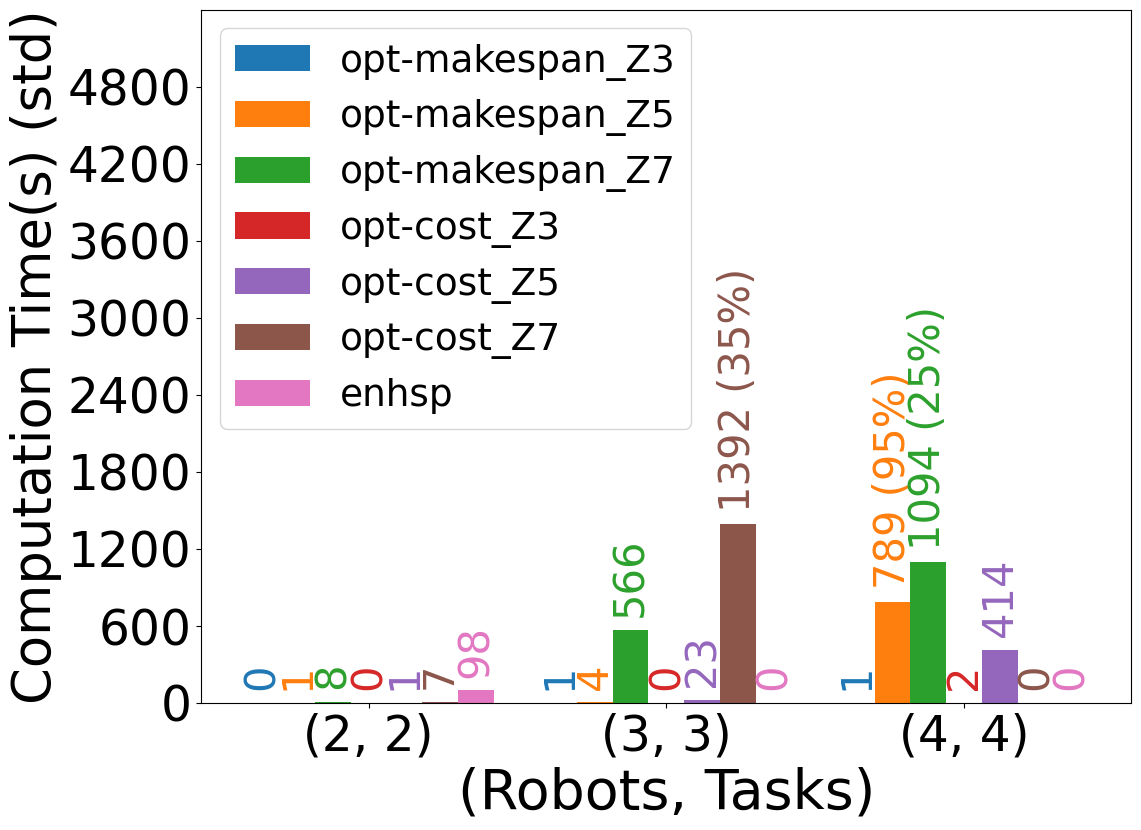}
} 
\caption{Comparison of various planners(legends) for varying Robots and Tasks with Collaboration on Computation Time. Mean (left) and Standard deviation (right).}
\label{fig:itmp_varyrt_collab_ct} 
\end{flushleft}
\end{figure}

\begin{figure}[t]
\begin{flushleft}
{
    \includegraphics[height=3.15cm, width=4.1cm]{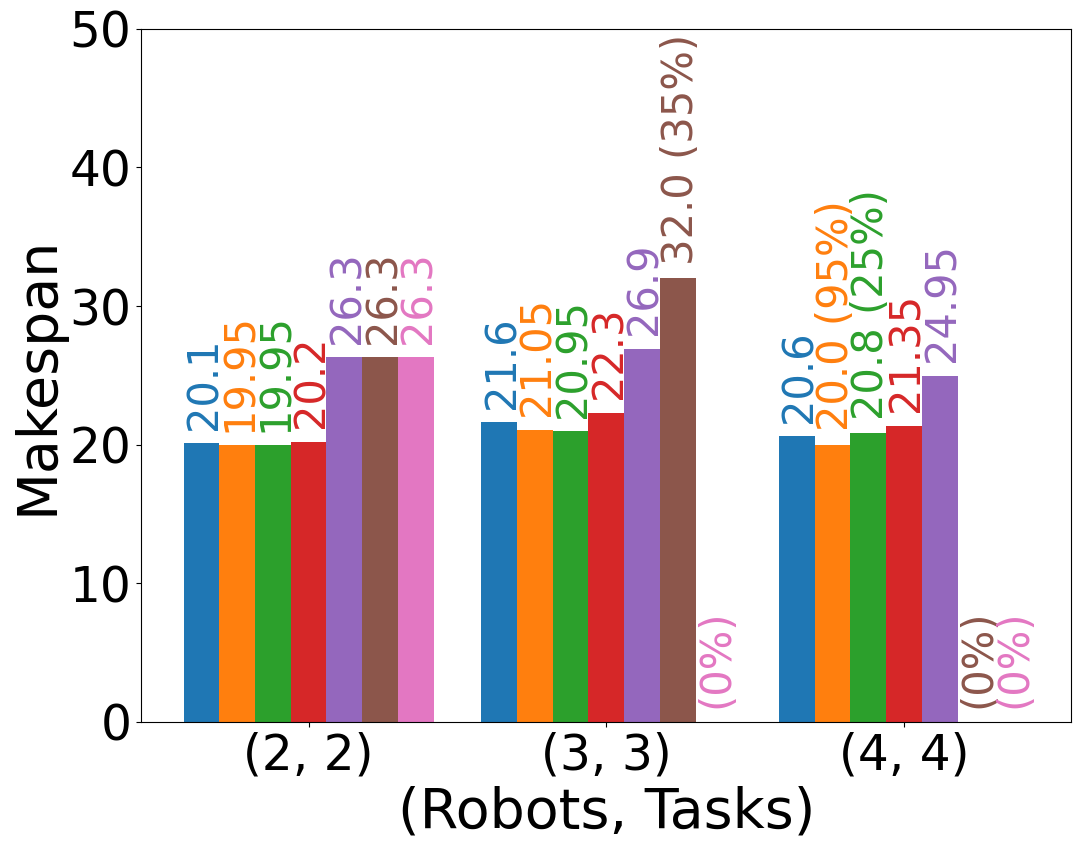}
}
{
    \includegraphics[height=3.15cm, width=4.1cm]{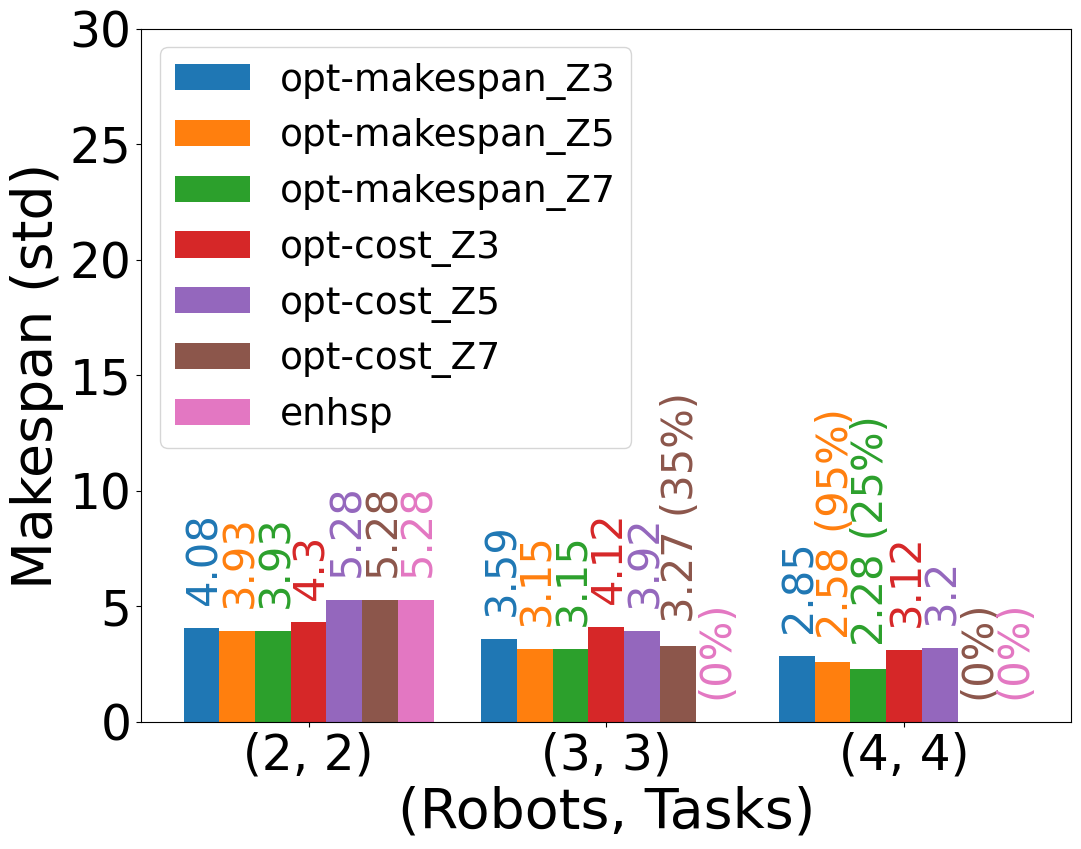}
} 
\caption{Comparison of various planners(legends) for varying Robots and Tasks with Collaboration on Makespan. Mean (left) and Standard deviation (right).}
\label{fig:itmp_varyrt_collab_ms} 
\end{flushleft}
\end{figure}

\begin{figure}[t]
\begin{flushleft}
{
    \includegraphics[height=3.15cm, width=4.1cm]{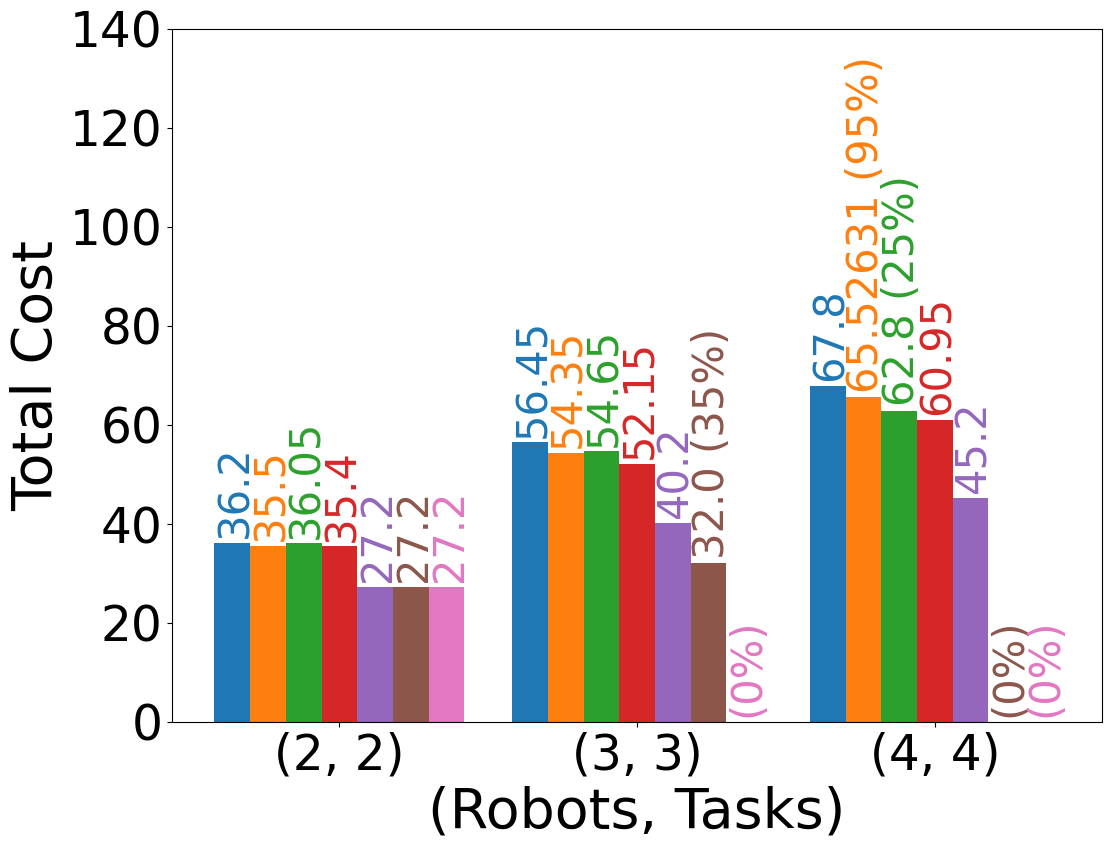}
}
{
    \includegraphics[height=3.15cm, width=4.1cm]{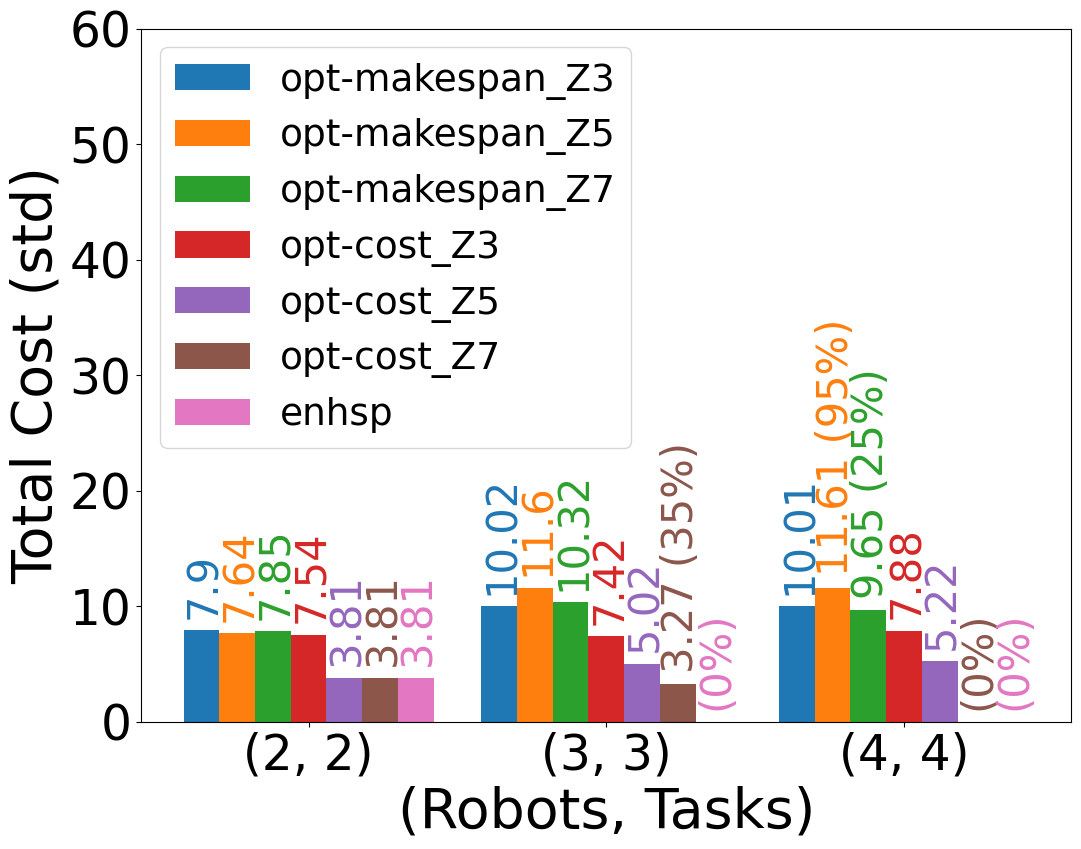}
} 
\caption{Comparison of various planners(legends) for varying Robots and Tasks with Collaboration on Total Cost. Mean (left) and Standard deviation (right).}
\label{fig:itmp_varyrt_collab_tc} 
\end{flushleft}
\end{figure}

\subsubsection{Comparison for varying Robots and Tasks with Collaboration}
We perform these experiments with a setup similar to the previous one, but we add some intermediate locations in the maps (randomly for randomly generated maps and predefined for predefined maps). We execute the planner with both the optimization criteria for multiple values of $Z$. We label our planner as $\mathtt{opt-makespan\_ZN}$ and $\mathtt{opt-cost\_ZN}$ in the plots, where $N$ denotes the value of $Z$. A value of $Z$=3 implies no collaboration; with a higher value of $Z$, the opportunity for intermediate pickup and drop arises. 
Figure~\ref{fig:itmp_varyrt_collab_ct} represents the computation times for various numbers of robots and tasks, and $Z$. For each robot and task, the computation time increases drastically for each increase in $Z$ for our planner. Our planner cannot solve all the problems for 4 robots and 4 tasks with $Z=7$. However, our planners are able to solve more problems faster compared to the classical planner. Figure~\ref{fig:itmp_varyrt_collab_ms} and Figure ~\ref{fig:itmp_varyrt_collab_tc} show the change in makespan and total cost for varying numbers of robots and tasks. Higher $Z$ values improve makespan for makespan optimization and total cost for total cost optimization. Also, our planners are able to generate better or equivalent plans compared to the classical planner.

\subsubsection{Additional Results}

\begin{figure}[t]
\begin{flushleft}
{
    \label{fig:itmp_maxrt_ct}
    \includegraphics[height=3.15cm, width=4.1cm]{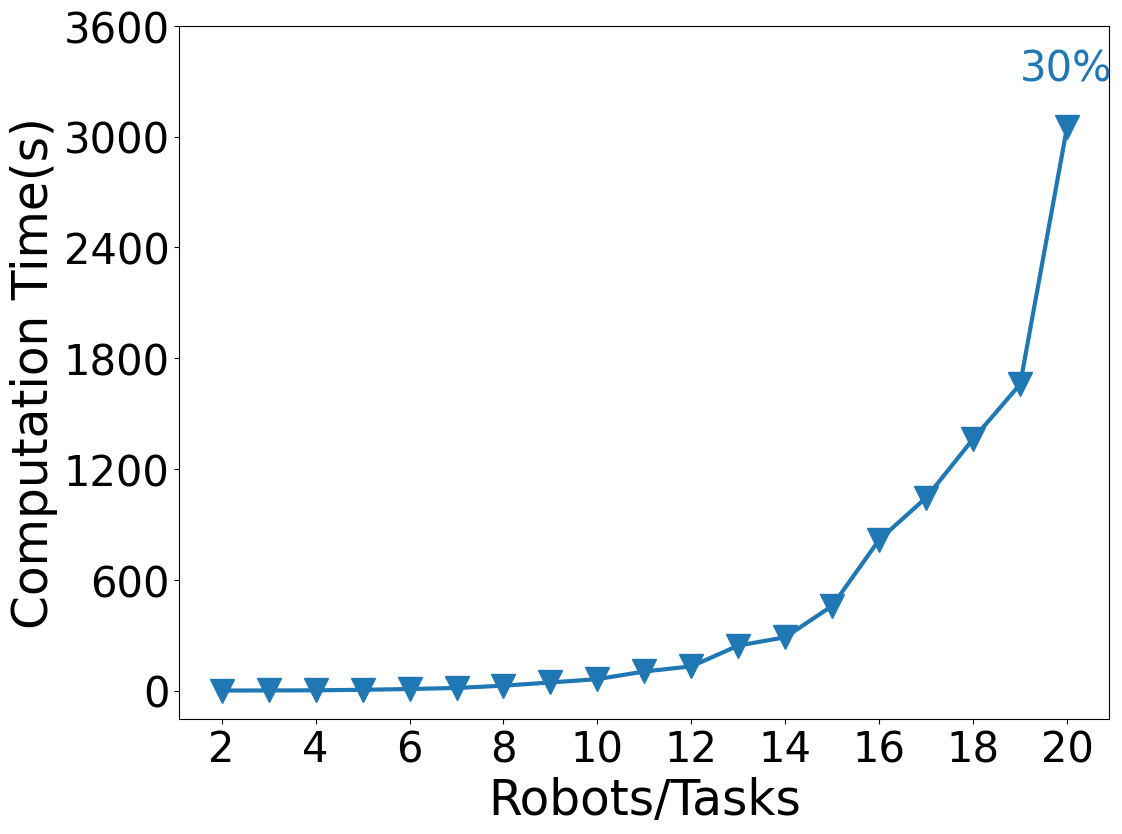}
}
{
    \label{fig:itmp_maxrt_ms}
    \includegraphics[height=3.15cm, width=4.1cm]{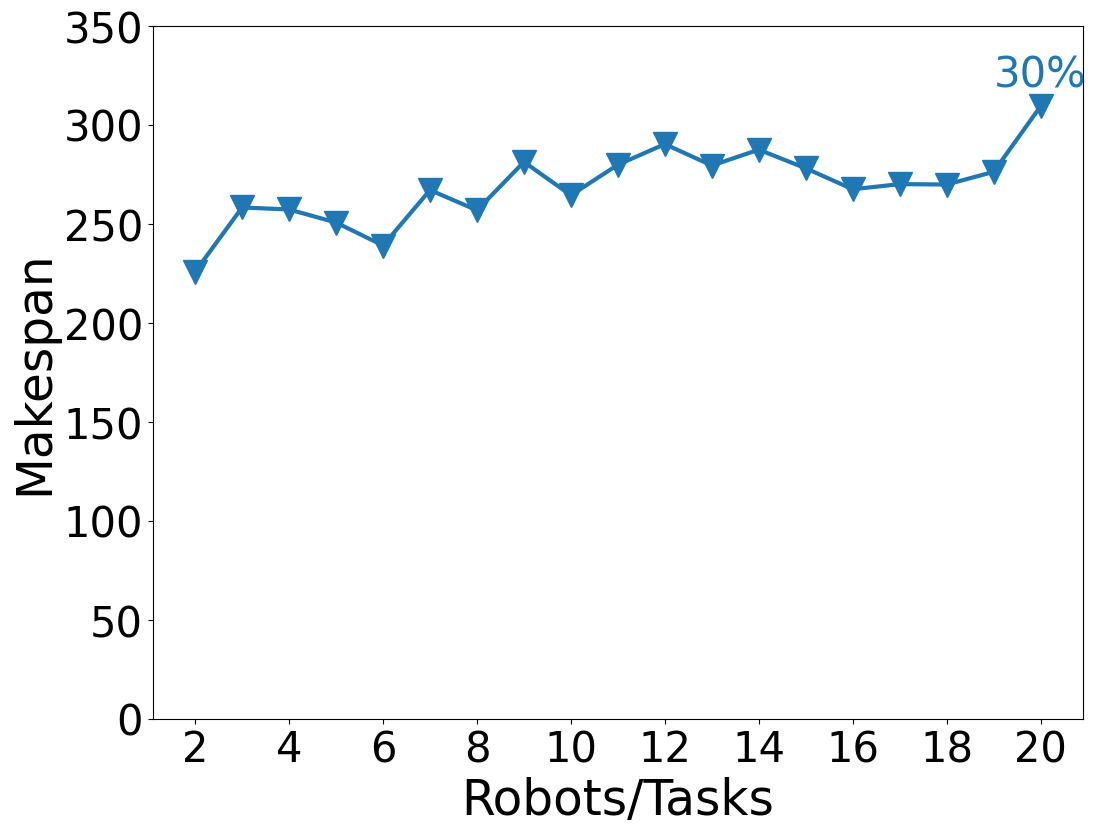}
} 
\caption{Varying Robots and Tasks without Collaboration for Integrated Task and Path Planner with makespan optimization criteria on Computation Time (left) and Makespan (right)}
\label{fig:itmp_maxrt} 
\end{flushleft}
\end{figure}

We also evaluate our algorithm for $N$ robots and $N$ tasks, where $N$ ranges from 2 to 20 for a $100 \times 100$ workspace to determine the scalability of our algorithm. Figure ~\ref{fig:itmp_maxrt}a and ~\ref{fig:itmp_maxrt}b represents the computation time and makespan for varying number of robots and tasks. Our planner can successfully execute upto 19 robots with 19 tasks without experiencing failures for a timeout of 3600s. We also evaluated the time distribution between task and path planner. On an average, the task planner consumes more than 98$\%$ of the total computation time. As the task planner explores a large search space to find the sequence of actions, the combinatorial explosion of possibilities makes the search exponentially large.  
Note that, in the some plots representing makespan, for some cases the average makespan for the ENHSP planner is slightly less than our planner. This is due to the fact that they are accumulated from the solved instances only, which are less in number.

}
\shortversion{\section{Evaluation}
We evaluate our planning methodology on various instances of warehouse pick-and-drop application scenarios.

\subsection{Experimental Setup}

For all our experiments, we use a desktop computer with an i7-4770 processor with 3.90\,GHz frequency and 12\,GB of memory. We use \textsf{Z3} SMT solver~\cite{Z3_Moura} from Microsoft Research to solve task-planning problems. For MA*-CBS-PC, we adapt the C++ code provided for \cite{zhang2022multi} with appropriate modifications.
The source code of our implementation is available at \url{https://github.com/iitkcpslab/Opt-ITPP}.

For any data point, we take the average of the results for multiple generated scenarios where the initial location of the robots and the task locations are generated randomly.
For each experiment, we create 20 different problem instances using predefined as well as randomly generated maps and report the mean result. We present an example for both types of maps in Section V-A and detailed results with standard deviations in Section V-C of~\cite{optITMPjournal}.


In our experiments, we consider two planners: one optimizes the makespan (opt\_makespan), and the other optimizes the total cost (opt\_cost). We compare our planners with a state-of-the-art classical planner ENHSP-20~\cite{scala2020subgoaling} as it supports numeric values required for capacities and deadlines and provides optimal solutions in terms of total cost.
We explored the possibility of modeling our problem as a constrained TSP problem and utilizing the meta-heuristic algorithm LKH3~\cite{helsgaun2017extension} to get a near-optimal solution. 
However, we did not find any extension of LKH3 that can deal with all the constraints we consider in our problem.
On the other hand, it was quite straightforward to model our exact problem in SMT as well as in ENHSP-20.

In our result plots, in all the instances where time is $3600\si{\second}$, the planner experiences a timeout. We include success percentages as annotations wherever the planner could not solve all the problems. In the plots, for all the cases where the planner faces a timeout,
we take its computation time as $3600\si{\second}$ and the metric value as the average of the values for the instances the planner can solve successfully.

\subsection{Results}

\begin{figure*}[t]
\begin{center}
{
    \label{fig:itmp_varyW_ct}
    \includegraphics[scale=0.175]{images/ITMP_PLOTS/VARYW/ComputationTime_nolegend.png}
}
{
    \label{fig:itmp_varyw_ms}
    \includegraphics[scale=0.175]{images/ITMP_PLOTS/VARYW/Makespan_nolegend.png}
}
{
    \label{fig:itmp_varyw_ms}
    \includegraphics[scale= 0.175]{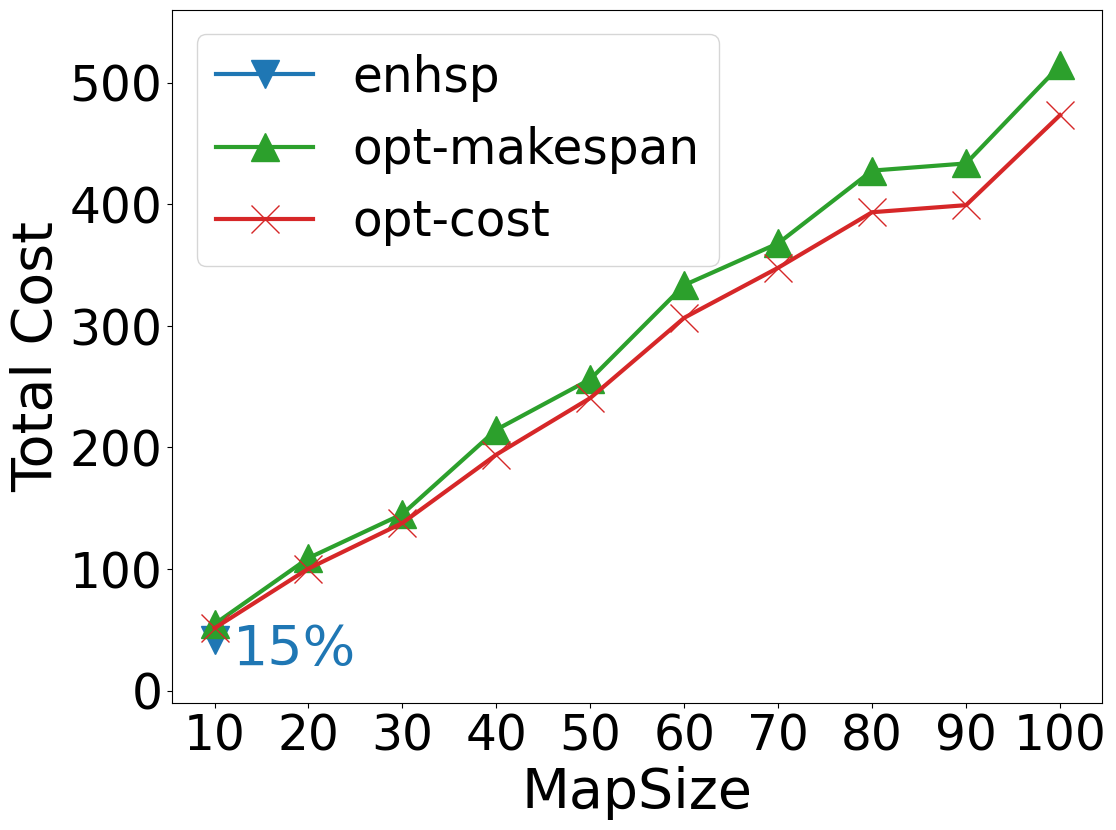}
}
\caption{Comparison of various planners for varying workspace size on (a) Computation Time (left), (b) Makespan (middle) and *c) Total Cost (right)}
\label{fig:itmp_varyw} 
\end{center}
\end{figure*}



\subsubsection{Comparison for varying workspace size}
In this evaluation, we experiment with $2$ robots and $2$ tasks with $Z=5$ for varying workspace sizes ranging from $10 \times 10$ to $100 \times 100$. Figure~\ref{fig:itmp_varyw}a shows the computation time for varying map sizes for our planners and  ENHSP-20 planner. Both our planners solve all the problems in  a few seconds. The ENHSP planner can solve $15\%$ of the problems for the smallest $10 \times 10$ map and is unable to solve any problem with a larger map size. Figures~\ref{fig:itmp_varyw}b and ~\ref{fig:itmp_varyw}c, presenting the makespan and total cost for all the three planners respectively, are as per the expectations, showing a linear increase in makespan and total cost with increase in map size.


\begin{figure*}[t]
\centering
{
    \label{fig:itmp_varyrt_ct}
    \includegraphics[scale=0.175]{images/ITMP_PLOTS/VARYRT/ComputationTime_nolegend.png}
}
{
    \label{fig:itmp_varyrt_ms}
    \includegraphics[scale=0.175]{images/ITMP_PLOTS/VARYRT/Makespan_nolegend.png}
} 
{
    \label{fig:itmp_varyrt_ms}
    \includegraphics[scale=0.175]{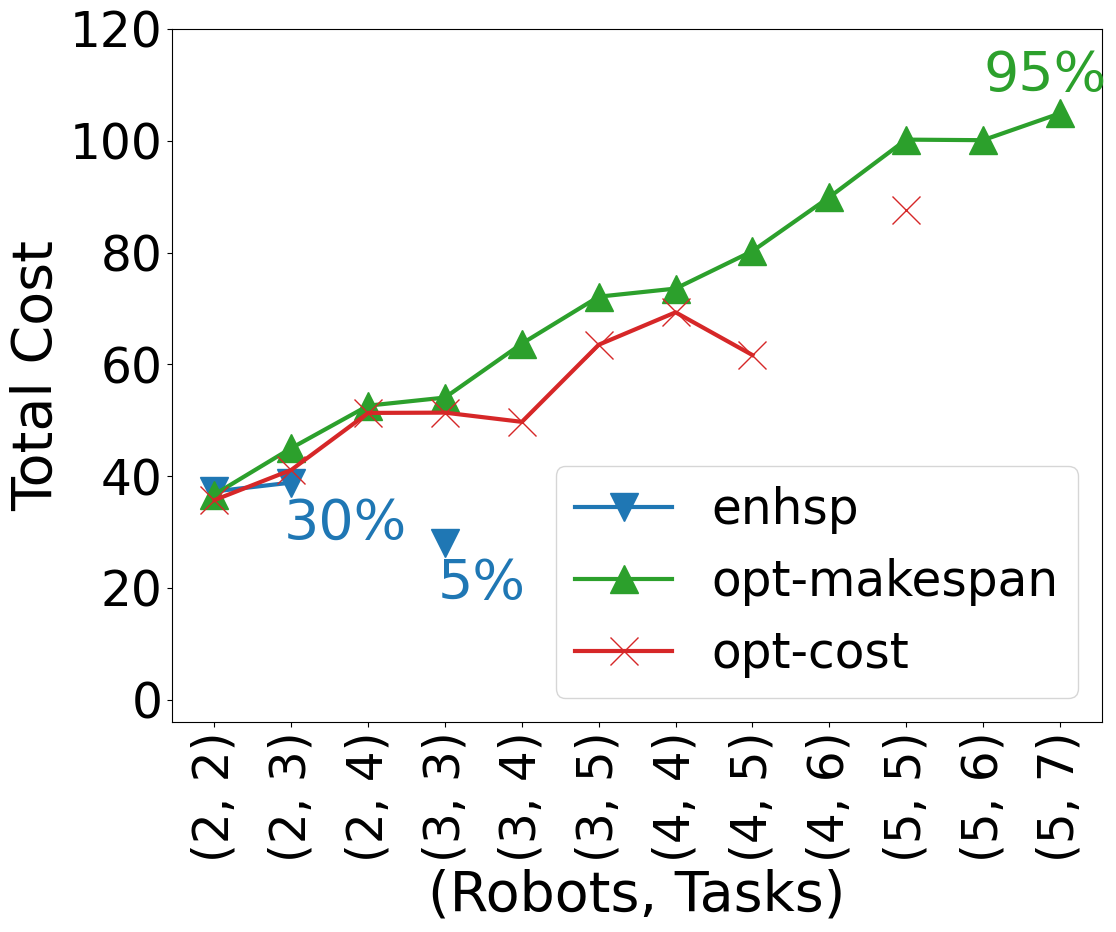}
} 
\caption{Comparison of various planners for varying number of Robots and Tasks without Collaboration on a) Computation Time (left), b) Makespan (middle) and c) Total Cost (right)}
\label{fig:itmp_varyrt} 
\end{figure*}



\subsubsection{Comparison for varying Robots and Tasks without Collaboration}
From the previous evaluation, we observe that the classical planners cannot solve problems for map size more than $10 \times 10$. So, in this experiment, we use maps of size  $9 \times 9$. 
We experiment with 2 to 5 robots and the number of tasks ranging from 2 to 7. Since we aim for a load-balanced solution, we use a minimum satisfiable $Z$ as it forces every robot to perform some work. 
Figure~\ref{fig:itmp_varyrt}a shows the computation time for different numbers of robots and tasks. The classical planner cannot solve any problem for more than 3 robots. Even for 3 robots, it can solve only some of the problem instances. Our planners perform significantly better compared to the classical planner. As optimizing total cost is harder for our planner, it starts facing timeout for 6 tasks. Our planner with makespan optimization can solve almost all of the problems. It faces timeout for 5\% of the cases for 5 robots and tasks.
Figures~\ref{fig:itmp_varyrt}b and ~\ref{fig:itmp_varyrt}c show how the makespan and the total cost vary with varying numbers of robots and tasks.


\begin{figure*}[t]
\begin{center}
{
    \label{fig:itmp_varyrt_collab_ct}
    \includegraphics[scale=0.175]{images/ITMP_PLOTS/VARYRTCOLLAB/ComputationTime_nolegend.png}
}
{
    \label{fig:itmp_varyrt_collab_ms}
    \includegraphics[scale=0.175]{images/ITMP_PLOTS/VARYRTCOLLAB/Makespan_nolegend.png}
} 
{
    \label{fig:itmp_varyrt_collab_ms}
    \includegraphics[scale=0.175]{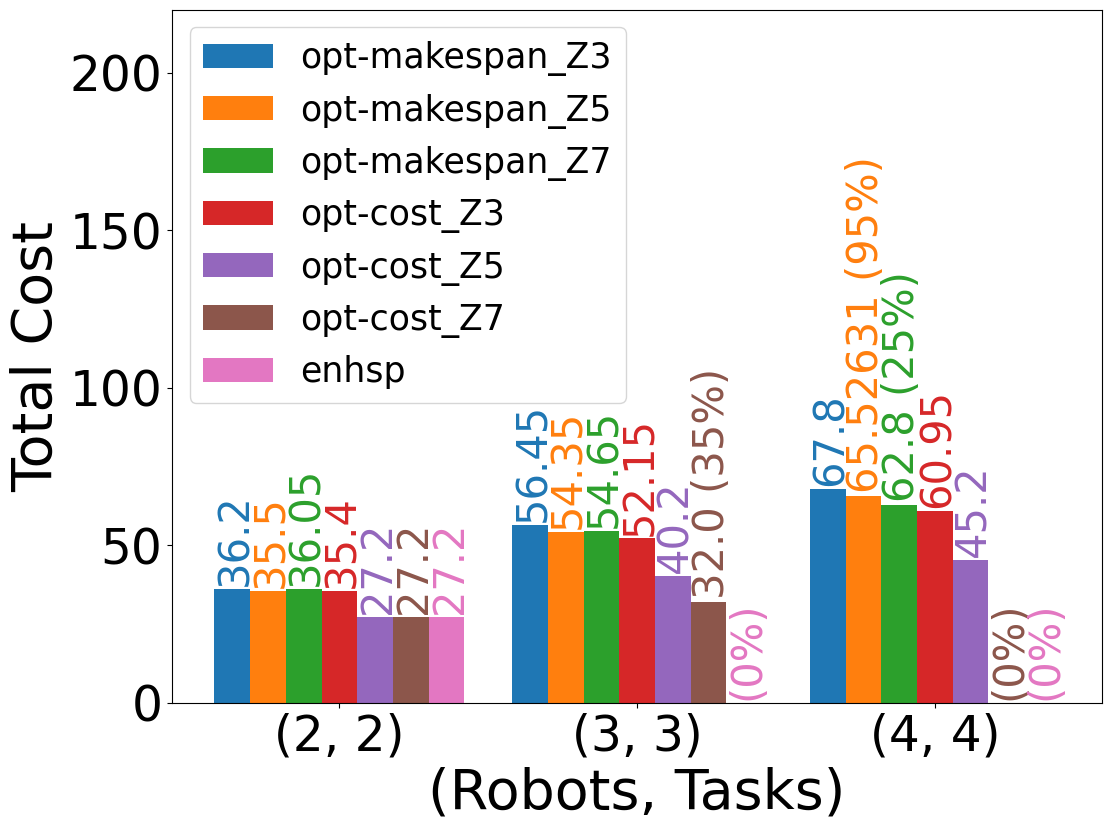}
}
\caption{Comparison of various planners(legends) for varying number  of Robots and Tasks with Collaboration on a) Computation Time (left), b) Makespan (middle) and c) Total Cost (right)}
\label{fig:itmp_varyrt_collab} 
\end{center}
\end{figure*}



\subsubsection{Comparison for varying Robots and Tasks with Collaboration}
We perform these experiments with a setup similar to the previous one, but we add some intermediate locations in the maps (randomly for randomly generated maps and predefined for predefined maps). We execute the planner with both the optimization criteria for multiple values of $Z$. We label our planner as $\mathtt{opt-makespan\_ZN}$ and $\mathtt{opt-cost\_ZN}$ in the plots, where $N$ denotes the value of $Z$. A value of $Z$=3 implies no collaboration; with a higher value of $Z$, the opportunity for intermediate pickup and drop arises. 
Figure~\ref{fig:itmp_varyrt_collab}a represents the computation times for various numbers of robots and tasks, and $Z$. For each robot and task, the computation time increases drastically for each increase in $Z$ for our planner. Our planner cannot solve all the problems for 4 robots and 4 tasks with $Z$=7. However, our planners are able to solve more problems faster compared to the classical planner. Figures ~\ref{fig:itmp_varyrt_collab}b and ~\ref{fig:itmp_varyrt_collab}c show the change in makespan and total cost for varying numbers of robots and tasks. Higher $Z$ values improve makespan for makespan optimization and total cost for total cost optimization. Also, our planners are able to generate better or equivalent plans compared to the classical planner.

\subsubsection{Additional Results}

We present a detailed evaluation of our Task Planner and some additional experimental results for our Integrated Planner in \cite{optITMPjournal}. From the evaluation of our Task Planner, we observe that computation time increases almost linearly with an increase in map size. But for varying robots and tasks, for each increase in Z, we observe a major increase in computation time. Our Task Planner performs better for makespan optimization criteria compared to the total cost. We evaluate our algorithm for N robots and  N tasks, where N ranges from 2 to 20 for a $100 \times 100$ workspace. Our planner with makespan optimization can successfully execute up to 19 robots with 19 tasks without experiencing a timeout ($3600\si{\second}$). We also evaluated the time distribution between the task and path planner. On average, the task planner consumes more than 98$\%$ of the total computation time. 
As the task planner explores a large search space to find the sequence of actions, The combinatorial explosion of possibilities makes the search challenging.

Note that, for some cases, the average makespan and total cost for the ENHSP planner is slightly less than our planners. This is due to the fact that they are accumulated from the solved instances only, which are less in number.

}
\section{Related Work}
\label{sec-related}

In this section, we briefly describe the related work in the domain of task and path planning for multi-robot systems. 
Several classical planners have been developed to solve task planning problems described in the popular multi-agent task specification language  MA-PDDL~\cite{Kovacs12}.
Leofante et al.~\cite{LeofanteANLT17} proposed an SMT-based mechanism to solve the multi-robot task scheduling problem in a logistic planning scenario that focuses on a simple objective involving only one state variable. In contrast, our SMT formulation considers multiple state variables for the robots and tasks to make it generic to handle complex scenarios.

Many previous papers have addressed the multi-agent path finding problem. 
Two prominent algorithms use A* search algorithm~\cite{HartNR68} for individual agents and rely on subdimensional expansion ($\textsf{M*} $~\cite{WagnerC11}) or constraint search tree (\textsf{CBS} ~\cite{sharon2015conflict}) to generate collision-free paths. Another approach with the SMT solver's capability to generate an unsatisfiable core is utilized to assign priorities to the robots to avoid any potential deadlock situation ~\cite{SahaR0PS16}. All these papers rely on task assignments from some other algorithm.

Several authors have presented algorithmic solutions for finding optimal task assignments and the corresponding collision-free paths for multi-robot applications.
Concurrent goal assignment and planning problem has been addressed by Turpin et al. for obstacle-free environments~\cite{TurpinMK14} and in the environment cluttered with obstacles~\cite{TurpinMMK14} without a guarantee of optimality.
On the other hand, the optimal goal assignment and the collision-free path-finding problem have been addressed in~\cite{MaK16,honig2018conflict,brown2020optimal}.
Recently, Okumura and D{\'e}fago have proposed a sub-optimal but fast algorithm for simultaneous target assignment and path planning efficiently for a large-scale multi-robot system.
Though the goal assignment is a form of task assignment, it is beyond the scope of these algorithms to deal with complex constraints (e.g., payload capacity, task deadline) for the robots or the possibility of robot-robot collaboration. 
Though the problem of transferring payloads in packet transfers~\cite{DBLP:conf/aaai/MaTSKK16} and deadline-aware planning~\cite{DBLP:conf/ijcai/0001WFLKK18} in a multi-agent environment have been studied, the proposed solutions apply to the very specific problems.
Several authors have presented mechanisms to solve the integrated task and path planning problem for multi-robot systems, where the task specifications are given using linear temporal logic~\cite{UlusoySDBR13, KantarosZ20,GujarathiS22}. These methods are either not scalable~\cite{UlusoySDBR13} or compromise on finding collision-free paths to achieve scalability~\cite{KantarosZ20,GujarathiS22}.


Several researchers have focused on the multi-robot pickup and delivery problem. Michal et al.~\cite{Michal15} provides a distributed algorithm to solve a well-formed multi-agent pickup-delivery problem. Ma et al.~\cite{MaLKK17, liu2019task} provide several algorithms addressing the MAPD problem across online and offline contexts. These approaches perform path planning in two stages, resulting in sub-optimal collision-free trajectories. Our approach employs CBS-PC~\cite{zhang2022multi}, which efficiently computes optimal collision-free trajectory. Though we take the pickup-delivery problem as an application, our SMT-based approach is more general in dealing with many complex constraints in a task planning problem. Some approaches based on Large Neighborhood Search ~\cite{xu2022multi, chen2021integrated} are efficient and scalable. 
However, these algorithms do not guarantee optimality or completeness; in contrast, our approach is complete and optimal. 


\section{Conclusion}
\label{sec-conclusion}

We have presented a generic integrated task and path planning algorithm for multi-robot systems and demonstrated the applicability of this framework on the pickup delivery problem that is at the core of any automated warehouse management system. 
Our planning framework provides an opportunity to combine the strength of an optimal task planner and an optimal path planner to design an optimal planner capable of solving complex multi-robot logistics planning problems which is beyond the scope of the  state-of-the-art multi-agent classical planners.

\longversion{
\section*{ACKNOWLEDGMENT}
\label{sec-acknowledgements}

This research was supported by Max-Plank Society, Germany through a research funding awarded to a partner group between MPI-SWS, Germany and IIT Kanpur, India.

}


\bibliography{references-short}

\end{document}